\def\isarxiv{1}

\ifdefined\isarxiv
\documentclass[11pt]{article}
\usepackage[numbers]{natbib}
\else
\documentclass{article}
\fi

\ifdefined\isarxiv

\else 

\usepackage[utf8]{inputenc} %
\usepackage[T1]{fontenc}    %
\usepackage{nicefrac}       %
\usepackage{microtype}      %
\usepackage{xcolor}         %
\fi

\ifdefined\isarxiv
\usepackage{tikz}

\usetikzlibrary{arrows}
\usepackage[margin=1in]{geometry}

\else
\definecolor{mydarkblue}{rgb}{0,0.08,0.45}
\fi

\usepackage{amsthm}
\usepackage{amssymb}
\usepackage{mathtools}
\usepackage{hyperref}
\usepackage[nameinlink,capitalize]{cleveref}
\hypersetup{colorlinks=true,linkcolor=blue,citecolor=blue,urlcolor=blue,pdfborder={0 0 0}}
\AddToHook{cmd/appendix/before}{%
  \crefalias{section}{appendix}%
  \crefalias{subsection}{subappendix}%
  \crefalias{subsubsection}{subsubappendix}%
}
\usepackage[normalem]{ulem} %
\usepackage{mathtools}

\usepackage[utf8]{inputenc} %
\usepackage[T1]{fontenc}    %
\usepackage{url}            %
\usepackage{booktabs}       %
\usepackage{amsfonts}       %
\usepackage{nicefrac}       %
\usepackage{microtype}      %
\usepackage{color, colortbl}
\definecolor{Gray}{gray}{0.9}
\newcolumntype{g}{>{\columncolor{Gray}}c}

\definecolor{Gray}{gray}{0.93}
\newcolumntype{a}{>{\columncolor{Gray}}c}

\usepackage{graphicx,wrapfig}
\usepackage{algorithm}
\usepackage[noend]{algpseudocode}
\usepackage{caption}
\usepackage{amsfonts}
\usepackage{url}
\usepackage{enumitem}
\usepackage{amsthm}
\usepackage{thmtools}
\usepackage{thm-restate}
\newtheorem{theorem}{Theorem}[section]
\newtheorem{corollary}{Corollary}[section]
\newtheorem{lemma}{Lemma}[section]
\newtheorem{assumption}{Assumption}[section]

\theoremstyle{definition}

\theoremstyle{remark}
\newtheorem{remark}{Remark}

\DeclareMathOperator*{\argmin}{arg\,min}

\newcommand{\bc}{\begin{center}}
\newcommand{\ec}{\end{center}}

\newcommand{\bdm}{\begin{displaymath}}
\newcommand{\edm}{\end{displaymath}}

\newcommand{\beq}{\begin{equation}}
\newcommand{\eeq}{\end{equation}}

\newcommand{\bfl}{\begin{flushleft}}
\newcommand{\efl}{\end{flushleft}}

\newcommand{\bt}{\begin{tabbing}}
\newcommand{\et}{\end{tabbing}}

\newcommand{\beqn}{\begin{align}}
\newcommand{\eeqn}{\end{align}}

\newcommand{\beqs}{\begin{align*}} %
\newcommand{\eeqs}{\end{align*}}  %

\usepackage[utf8]{inputenc}
\usepackage[T1]{fontenc}
\usepackage{url}
\usepackage{booktabs}
\usepackage{amsfonts}
\usepackage{nicefrac}
\usepackage{microtype}
\usepackage{xcolor}
\usepackage{amsmath,amssymb,amsthm,mathtools,bm}
\usepackage{multirow,array,enumitem}
\usepackage{algorithm}
\usepackage[noend]{algpseudocode}
\usepackage{graphicx}
\usepackage{rotating}
\usepackage{xspace}
\usepackage{subcaption}
\usepackage{natbib}
\usepackage[normalem]{ulem}
\usepackage{comment}
\usepackage{etoc}

\newcommand{\R}{\mathbb{R}}
\newcommand{\E}{\mathbb{E}}

\crefname{assumption}{Assumption}{Assumptions}
\Crefname{assumption}{Assumption}{Assumptions}
\crefname{remark}{Remark}{Remarks}
\Crefname{remark}{Remark}{Remarks}
\crefname{corollary}{Corollary}{Corollaries}
\Crefname{corollary}{Corollary}{Corollaries}

\newcommand{\set}[1]{\{#1\}}

\title{Dysco: Dynamic Subspace Boosting to Mitigate LoRA Interference in Federated Learning}

\author{%
  Haobo Zhang\\
  University of Michigan\\
  \texttt{haobozha@umich.edu} \\
  \and
  Jiankun Wang\\
  University of Michigan\\
  \texttt{jiankun@umich.edu}\\
  \and
  Suraj Rajendran\\
  Cornell University\\
  \texttt{sur4002@med.cornell.edu}\\
  \and
  Weishen Pan\\
  Cornell University\\
  \texttt{wep4001@med.cornell.edu}\\
  \and
  Lam Tsoi \\
  University of Michigan \\
  \texttt{alextsoi@umich.edu}
  \\
  \and
  Yong Chen\\
  University of Pennsylvania\\
  \texttt{ychen123@upenn.edu}\\
  \and
  Fei Wang\\
  Cornell University\\
  \texttt{few2001@med.cornell.edu}\\
  \and
  Jiayu Zhou\\
  University of Michigan\\
  \texttt{jiayuz@umich.edu}
}

\date{}

\begin{document}
\maketitle

\begin{abstract}
Federated fine-tuning of large pre-trained models increasingly relies on Low-Rank Adaptation (LoRA) to reduce communication and computation, but heterogeneous clients can make adapter aggregation unstable.
We identify a geometric source of this instability. After LoRA adapters are merged, an adapter trained on one client is also applied to the representations of other clients, producing \emph{data-parameter interference}.
This interference is controlled by the alignment between LoRA update subspaces and client activations, suggesting that federated LoRA aggregation should be viewed not only as parameter averaging but also as \emph{subspace allocation}.
We propose \textbf{Dy}namic \textbf{S}ubspa\textbf{c}e B\textbf{o}osting \textbf{(Dysco)}, a plug-in method that allocates client-specific LoRA subspaces in a federated and dynamic manner.
In each round, clients compute activation-insensitive subspaces from local representations and transmit only the resulting bases; the server then constructs client-specific merged subspaces through a closed-form solution that maximizes compatibility with other clients' insensitive directions.
To handle representation drift, Dysco performs multi-round subspace boosting to preserve past update directions while adapting to future representations.
We provide a convergence analysis that embeds the data-parameter interference as an aggregation-error term in a standard federated optimization bound, and prove that Dysco's server-fixed merged subspaces yield a tighter upper bound on this error.
Experiments on controlled synthetic federated tasks and on MIMIC-IV clinical-note classification with Llama-3.2-1B show that Dysco substantially reduces interference, reduces the final-round synthetic training loss by up to $9$ times relative to baselines under the orthogonal-subspace partition the theory identifies, improves all five tested FL algorithms by up to $4.3\%$ on MIMIC, outperforms recent federated LoRA methods, and adds only $0.9\%$ wall-clock overhead.
Our code is available at \url{https://github.com/illidanlab/Dysco}.
\end{abstract}

\section{Introduction}
\label{sec:intro}

Federated learning (FL)~\citep{fedavg} enables multiple clients to jointly train a model without centralizing raw data, making it attractive for privacy-sensitive applications and for the distributed adaptation of large language models.
As foundation models grow, communicating and locally updating full model weights becomes prohibitive for FL clients, so parameter-efficient fine-tuning has become essential to federated adaptation.
Low-Rank Adaptation (LoRA)~\citep{lora} is the dominant choice, where each client trains and communicates only a low-rank update $\Delta W = BA$ instead of full model weights, reducing both local compute and communication cost by orders of magnitude.

LoRA-based FL, however, suffers heavily from data heterogeneity.
When clients hold different data distributions or task semantics, their local updates conflict after aggregation, causing client drift, unstable optimization, and poor global performance.
Some literature addresses this problem by modifying the optimization dynamics through proximal regularization, momentum, normalized updates, or control variates~\citep{fedavg,fedprox,fednova,scaffold}, or by introducing auxiliary data, personalized models, and knowledge distillation~\citep{fedaux,fedftg,fedgen,feddf,IFCA,ditto,fedamp,perfedavg}.
Subspace constraints have also proven effective in centralized settings for continual learning, multi-task learning, parameter-efficient adaptation, and model editing~\citep{OGD,PCGrad,Null-LoRA,NSPO,alphaedit,sonoedit,OSRM,mingle}, and recent federated LoRA work studies how adapters should be personalized, partitioned, or aggregated under heterogeneity~\citep{yi2023pfedlora,qi2024fdlora,sun2024improving,guo2025selective,cho2024heterogeneous,bai2024federated,wang2024flora,koo2025towards}.
Most literature attempts to resolve conflicts among model parameters and align them before or after aggregating local models.

In this paper, however, we trace the degradation to a structural mechanism that these approaches leave unaddressed.
Aggregating LoRA adapters allows one client's adapter to act on another client's representations: for two clients aggregated by
$
W_{\mathrm{merge}} = W + B_1 A_1 + B_2 A_2,
$
evaluating the merged model on an activation $h_1$ from client~1 produces the cross-client term $B_2 A_2 h_1$, which applies client~2's adapter to client~1's features and can perturb client~1's predictions after aggregation.
We call this effect \emph{data-parameter interference}.
Importantly, the interference depends on the alignment between the row space of $A_2$ and the activation subspace of client~1, so mitigating it requires controlling the subspaces in which clients update.
No existing line of work exercises this control.
Optimization-level remedies manipulate the training dynamics and gradients, while federated LoRA designs reorganize the adapters to be shared. 
Centralized subspace methods propose to constrain update subspaces, but they either violate the FL protocol or rely on a fixed subspace which can be stale during training.
Removing interference therefore calls for a new approach, in which each client should update in a subspace that is expressive for its own data yet insensitive to other clients' data.
Such subspace allocation must be \emph{federated}, using only local information, and \emph{dynamic}, adapting as representations evolve.

To address the data-parameter interference under the two requirements, we propose \textbf{Dy}namic \textbf{S}ubspa\textbf{c}e Bo\textbf{o}sting \textbf{(Dysco)}, a plug-in method that treats federated LoRA aggregation as a subspace-allocation problem.
In each communication round, every client first computes an activation-insensitive subspace from its local hidden representations by solving an orthonormal projection problem with a closed-form solution.
Only the resulting basis is transmitted to the server, so raw data and intermediate activations remain local and the procedure stays federated.
The server then constructs a client-specific merged subspace by maximizing its overlap with the insensitive subspaces of the other clients.
Finally, to remain dynamic under representation drift, Dysco performs multi-round subspace boosting
to preserve useful directions from earlier rounds while adapting future updates to the current representation geometry.

Theoretically, we provide a convergence analysis by examining the data-parameter interference as an aggregation-error term, and show that Dysco yields a tighter bound on this error under data heterogeneity.
Empirically, we evaluate Dysco on controlled synthetic federated tasks and on MIMIC-IV clinical-note classification in the main body, and report additional GLUE benchmark experiments in the appendix.
On synthetic federated regression tasks under the orthogonal-subspace partition that the theory identifies, Dysco drives the training loss down substantially faster than FedAvg-style baselines, reducing the final-round training loss by a factor of up to $9$ across Transformer depths and LoRA ranks.
On MIMIC-IV clinical note classification~\citep{mimic-iv-note} with Llama-3.2-1B~\citep{llama3}, Dysco improves all five baselines, with a headline gain of $4.3$ points on FedAvgM~\citep{fedavgm} and a near-flat $0.5$ on Scaffold~\citep{scaffold} from a degenerate $50\%$ baseline.
It also outperforms recent federated LoRA baselines, including FFA-LoRA~\citep{ffa-lora} and FedSA-LoRA~\citep{fedsa-lora}, while adding only $0.9\%$ wall-clock overhead on average.
Consistent average-accuracy improvements across baselines are observed on the GLUE benchmark~\citep{glue} with RoBERTa-large~\citep{roberta}, with per-task variability reflecting the higher cross-task interference of an eight-task federation.

\paragraph{Contributions.}
We summarize our main contributions as follows:
\begin{itemize}
\item 
We show that aggregating LoRA adapters creates cross-client terms,
revealing a concrete mechanism by which heterogeneity destabilizes LoRA-based FL,
and formulate the model aggregation as a subspace-allocation problem.

\item 
We propose Dysco, which combines local activation-insensitive subspace selection, closed-form server-side subspace merging, and multi-round subspace boosting to adapt to representation drift.
We analyze the convergence of Dysco and prove that it yields a tighter upper bound on the aggregation error.

\item 
Through extensive experiments under synthetic settings and real benchmarks such as GLUE with RoBERTa-large and MIMIC-IV-note with Llama-3.2-1B, we show the significant improvement of global performance, fewer rounds to converge, and negligible computational overhead against baselines, across various client scales, LoRA ranks, and block selections.

\end{itemize}

\section{Preliminaries}
\label{sec:preliminaries}

\subsection{Federated Learning}
\label{subsec:problem}

We consider a federated learning setting with $N$ clients indexed by $i\in\{1,\dots,N\}$, trained over $R$ communication rounds indexed by $t\in\{0,\dots,R-1\}$.
Client~$i$ holds a dataset $\mathcal{D}_i$ drawn from its local distribution, which may differ substantially from other clients' data distributions.
The tasks can be different across clients, or they could be the same task but with differently distributed inputs.
We denote the model parameters as $W$.
Each client~$i$ has a population objective $f_i$, and the global objective is $f(W) = \frac{1}{N}\sum_{i=1}^{N} f_i(W)$, with minimum value $f^\star = \inf_W f(W)$.
In standard FL, all clients share the same model architecture $W$ and aim to learn a single global $W$ that is effective for all $\mathcal{D}_i$.

\subsection{Low-Rank Adaptation}
\label{subsec:lora}

In this paper, we mainly focus on training language models with low-rank adaptation (LoRA).
Rather than training the full model $W$ on each client, which can be large and costly to communicate, each client only learns a low-rank update to $W$.
Specifically, at round $t$ client $i$ holds trainable matrices $(B_i^t, A_i^t)\in\R^{d\times r}\times\R^{r\times d}$ for a small rank $r \ll d$, where $d$ is the dimension of the weight matrix, and modifies the model as $W + B_i^t A_i^t$.
Initially, $B_i^0 A_i^0$ is zero, and each client trains its LoRA factors on local data while keeping $W$ fixed.
This low-rank training drastically reduces communication, as each client only sends the $B_i^t$ and $A_i^t$ matrices of size $O(dr)$ instead of the full $W$ of size $O(d^2)$.
The canonical FedAvg-LoRA aggregation yields the global model $W + \frac{1}{N}\sum_{i=1}^NB_i^t A_i^t$ at the end of round $t$.
We suppress the round superscript and write $(B_i, A_i)$ when the round index is not essential.
In our approach, these LoRA matrices also serve as the foundation for coordinating subspaces.

\section{Our Method}
\label{sec:method}

\begin{figure}[t]
    \centering
    \includegraphics[width=.8\linewidth]{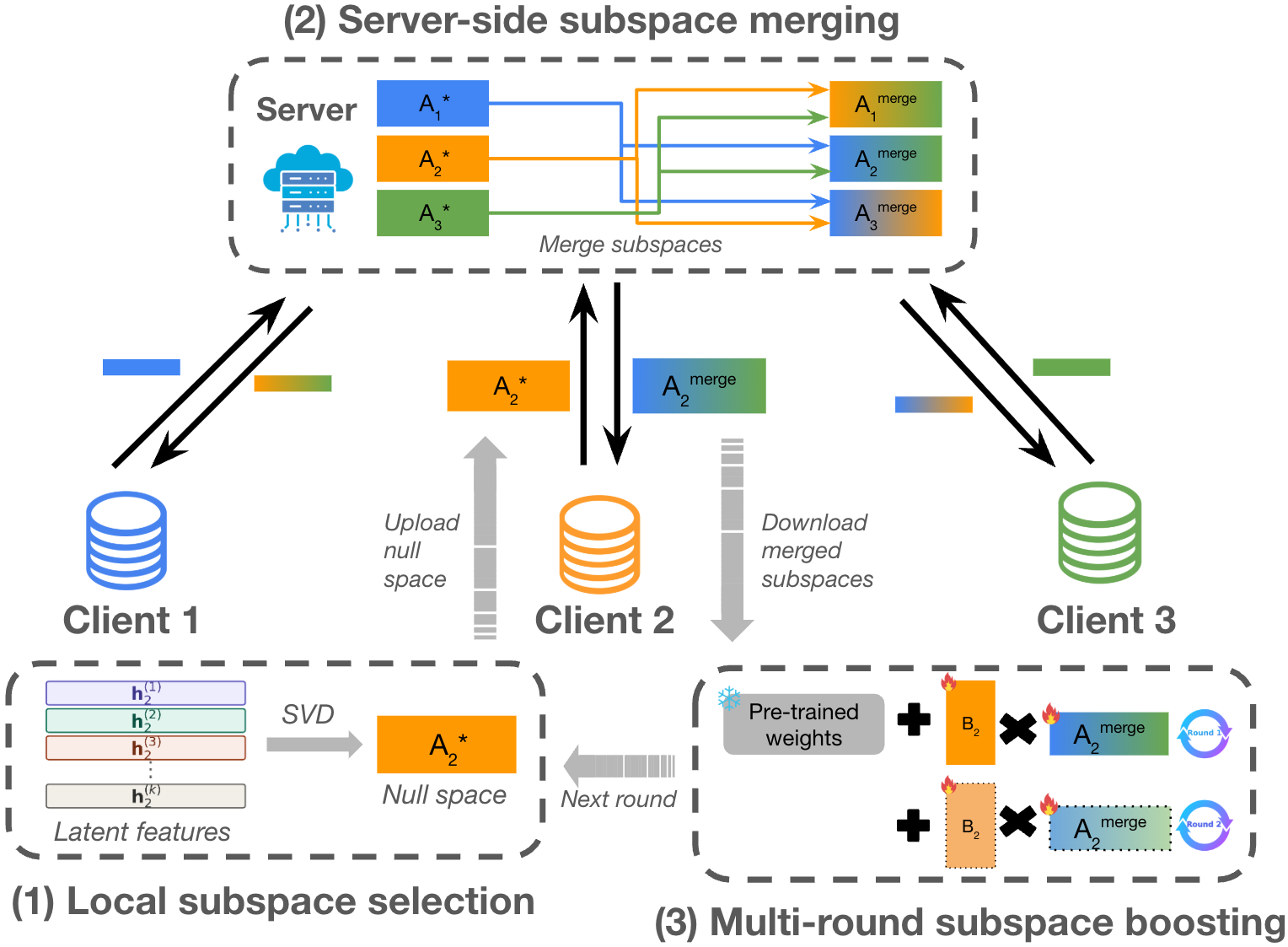}
    \caption{Overview of Dysco. (1) \emph{Local subspace selection}: each client runs SVD on its latent features to obtain a basis $A_i^*$ spanning directions insensitive to its own data, and uploads only this basis. (2) \emph{Server-side subspace merging}: the server combines the other clients' bases into a merged subspace $A^{\mathrm{merge}}$ for each client, allocating directions that are broadly insensitive to the other clients' representations. (3) \emph{Multi-round subspace boosting}: to track representation drift, each round refreshes local selection and server-side merging on the current features and stacks a fresh LoRA block on top of the frozen prior pair $(B, A^{\mathrm{merge}})$, growing the adapter as an additive ensemble of round-wise subspaces.}
    \label{fig:overview}
\end{figure}

We address data heterogeneity in LoRA-based federated learning by controlling the update subspaces in which client adapters operate.
Rather than first averaging client parameters and then correcting the resulting conflicts, we allocate client-specific subspaces that reduce cross-client interference during aggregation.
We begin by showing that interference arises when an adapter learned from one client acts on the representations of another client, which motivates a subspace-allocation view of federated LoRA (\cref{subsec:interference_nullspace}).
This view naturally suggests using activation-insensitive directions. However, making this principle work in FL is challenging because client representations cannot be pooled and the relevant subspaces evolve throughout training.
We then present \textbf{Dy}namic \textbf{S}ubspa\textbf{c}e B\textbf{o}osting \textbf{(Dysco)}, which implements subspace allocation in a federated and dynamic manner through local subspace selection, server-side subspace merging, and multi-round subspace boosting (\cref{subsec:federated_merging}).
Finally, we briefly summarize the theoretical analysis in~\cref{subsec:method-theory-summary}, deferring the full statement and proofs to~\cref{app:dysco_vs_fedavg}.
The complete algorithm is summarized in~\cref{alg:method}, and an overview is shown in~\cref{fig:overview}. 

\subsection{From Interference to Subspace Allocation}
\label{subsec:interference_nullspace}

\paragraph{Interference in LoRA Aggregation.}
We first identify a geometric source of negative transfer in LoRA-based FL.
Consider two clients whose LoRA adapters are aggregated into
\[
    W_{\mathrm{merge}} = W + B_1A_1 + B_2A_2 .
\]
When this merged model is evaluated on an activation $h_1$ from client~1, the forward computation decomposes as
\begin{align}
    W_{\mathrm{merge}}h_1
    &=
    (W + B_1A_1 + B_2A_2)h_1 \notag \\
    &=
    \underbrace{(W+B_1A_1)h_1}_{\text{client-1 adaptation}}
    +
    \underbrace{B_2A_2h_1}_{\text{cross-client interference}} .
    \label{eq:interference}
\end{align}
The second term applies the adapter learned from client~2 to client~1's representation.
When client distributions are heterogeneous, $A_2h_1$ need not be small, so the update from one client can alter the predictions of another in an uncontrolled way.
This effect is particularly pronounced in LoRA because the right factor $A_i$ selects the input directions along which client~$i$ is allowed to update the model.
For brevity, we omit layer indices unless otherwise specified.

\paragraph{Aggregation as Subspace Allocation.}
Equation~\eqref{eq:interference} suggests that the role of the server is not only to average client updates, but also to decide \emph{where} those updates should live.
The interference term disappears whenever the update subspace of client~2 is orthogonal to client~1's activations, namely $A_2h_1 \approx 0$, regardless of the learned left factor $B_2$.
Thus, controlling the row space of $A_i$ provides a direct mechanism for controlling cross-client interference.

We therefore view federated LoRA aggregation as a \emph{subspace allocation} problem.
For each client, the server should allocate an update subspace that is expressive for that client's local task while being insensitive to the representations of other clients.
Standard FedAvg-LoRA aggregation 
corresponds to the degenerate allocation in which every client may update the same ambient space, leaving the interference term in~\eqref{eq:interference} unconstrained.
Our goal is instead to allocate client-specific low-dimensional subspaces before local training, so that subsequent updates are less likely to activate other clients' representations.

\paragraph{Learning Insensitive Subspaces.}
The basic building block is an \emph{insensitive subspace}, a set of directions that produces minimal response on a client's activations.
Let $H_i \in \mathbb{R}^{n_i \times d}$ denote the matrix of latent representations collected from client~$i$, with each row corresponding to one activation vector.
For a target rank $r$, we compute an orthonormal basis
\begin{equation}
    \label{eq:compute subspace}
    A_i^*
    =
    \arg\min_{A \in \mathbb{R}^{r \times d}}
    \| A H_i^\top \|_F^2
    \qquad
    \text{s.t. } AA^\top = I_r .
\end{equation}
The objective measures the total energy of client~$i$'s activations after projection onto the row space of $A$.
Minimizing it selects directions to which client~$i$ is least sensitive, while the orthonormality constraint prevents the trivial zero solution and fixes the subspace dimension.

This problem has a closed-form solution.
Let
\[
    C_i = H_i^\top H_i
\]
be the empirical activation Gram matrix, and write its eigendecomposition as
$C_i = Q_i \Sigma_i Q_i^\top$, with eigenvalues
$\sigma_1 \ge \sigma_2 \ge \cdots \ge \sigma_d \ge 0$.
Then $A_i^*$ is given by the $r$ eigenvectors associated with the smallest eigenvalues of $C_i$.
Equivalently, $A_i^*$ spans the approximate null space of client~$i$'s representation distribution.
If another client's update is constrained to this subspace, then its right LoRA factor produces little response on client~$i$'s data, thereby reducing the interference term in~\eqref{eq:interference}.

\paragraph{Toward Federated and Dynamic Allocation.}
The insensitive-subspace criterion above suggests a natural way to suppress interference by updating the model along directions that are weakly activated by the data on which the update should have minimal effect.
In a centralized setting, such directions could be computed from pooled representations by constructing $H_{\neg i}$ from the activations of clients other than~$i$ and solving an objective of the form
\[
    \arg\min_A \|A H_{\neg i}^{\top}\|_F^2 .
\]
This centralized variant subsumes prior null-space-constrained methods such as Null-LoRA~\citep{Null-LoRA}, AlphaEdit~\citep{alphaedit}, and related work on knowledge editing and model merging~\citep{NSPO,OSRM}, which solve the same insensitive-subspace objective on pooled centralized activations and treat the resulting basis as fixed. 
In this sense, centralized null-space-LoRA is the static, data-pooled special case of the subspace-allocation problem we study here.
However, this direct construction is not suitable for federated learning.
First, it requires access to cross-client activations, whereas FL is designed to avoid sharing client data or intermediate representations.
Second, the relevant activation geometry is not fixed.
As local training and global aggregation update the model, the representations seen by each client also change, so a subspace computed once can gradually become misaligned.

These two constraints make subspace allocation in FL fundamentally different from its centralized counterpart.
The allocation must be \emph{federated}, using only locally computed information, and \emph{dynamic}, adapting as client representations drift across communication rounds.
In the next section, Dysco is proposed to realize this federated and dynamic form of subspace allocation.

\subsection{Dysco for Dynamic Subspace Boosting}
\label{subsec:federated_merging}

Dysco solves the subspace allocation problem in the federated and dynamic setting identified above.
It consists of three steps aligned with the two constraints of FL.
\emph{Local subspace selection} and \emph{server-side subspace merging} address the information constraint by allowing clients to compute insensitive subspaces locally and allowing the server to allocate subspaces using only transmitted bases, without accessing raw features or intermediate representations.
\emph{Multi-round subspace boosting} addresses the temporal constraint by recomputing and accumulating subspaces as representations evolve during federated training.

\paragraph{Local Subspace Selection.}
Each client~$i$ solves~\cref{eq:compute subspace} using only its local data and obtains an insensitive subspace $A_i^* \in \mathbb{R}^{r \times d}$.
The client then sends only the orthonormal basis $A_i^*$ to the server.
This preserves the FL protocol because no raw samples or intermediate activations are shared.
However, $A_i^*$ captures directions that are insensitive to client~$i$'s own data, while the update subspace assigned to a client should ideally be insensitive to the data of the other clients.
The server therefore needs to combine the locally computed subspaces into client-specific allocations.

\paragraph{Server-Side Subspace Merging.}
After collecting $\{A_1^*, A_2^*, \dots, A_N^*\}$ from all clients, the server constructs a merged subspace $A^{\mathrm{merge}}_i$ for each client~$i$.
The goal is to allocate to client~$i$ directions that align well with the insensitive subspaces of the other clients.
We formulate this merging step as
\begin{equation}
    \label{eq:merge_obj}
    A^{\mathrm{merge}}_i
    =
    \arg\max_{A \in \mathbb{R}^{r \times d}}
    \sum_{j=1;\, j \ne i}^{N}
    \| A A_j^{* \top} \|_F^2
    \qquad
    \text{s.t. } A A^\top = I_r .
\end{equation}
The quantity $\|A A_j^{*\top}\|_F^2 =
\operatorname{tr}(A A_j^{*\top} A_j^* A^\top)$ measures the overlap between the candidate subspace $A$ and the insensitive subspace of client~$j$.
Maximizing the aggregate overlap selects directions that are broadly compatible with the clients whose predictions should not be disturbed.

Problem~\eqref{eq:merge_obj} also admits a closed-form solution.
Define
\[
    S_i = \sum_{j=1;\,j \ne i}^N A_j^{* \top} A_j^*
    \in \mathbb{R}^{d \times d},
\]
which aggregates the projectors onto the insensitive subspaces of all clients except~$i$.
Let $S_i = V \Lambda V^\top$, where the eigenvalues satisfy
$\lambda_1 \ge \lambda_2 \ge \cdots \ge \lambda_d \ge 0$ and the corresponding eigenvectors are $V = [v_1, \dots, v_d]$.
The optimal merged subspace is given by the top-$r$ eigenvectors of $S_i$,
\begin{equation}
    \label{eq:solve merged subspace}
    A^{\mathrm{merge}}_i
    =
    [\, v_1^\top;\; v_2^\top;\; \cdots;\; v_r^\top \,],
\end{equation}
up to an arbitrary $r \times r$ rotation.
The server then broadcasts $A^{\mathrm{merge}}_i$ back to client~$i$.

The matrix $S_i$ summarizes how strongly each direction is supported by the other clients' insensitive subspaces.
When these subspaces share common directions, those directions become dominant eigendirections of $S_i$.
When the subspaces differ, the top eigenspace balances directions shared by many clients against directions that are strongly preferred by a smaller subset.

\paragraph{Multi-Round Subspace Boosting.}
\label{subsec:dynamic boosting}
After receiving $A^{\mathrm{merge}}_i$, each client initializes its right LoRA factor with the merged subspace and trains the left factor $B_i$ on local data.
Because the model changes after local training and server aggregation, the latent representations also change across communication rounds.
Let $H_i^t$ denote the feature matrix of client~$i$ at round~$t$.
Solving~\cref{eq:compute subspace} with $H_i^t$ yields a local subspace $A_{i,t}^*$, and the subsequent merging step produces $A_i^{\mathrm{merge},t}$.
After one round of training and aggregation, the feature matrix shifts to $H_i^{t+1}$, so the previously merged subspace may no longer match the current representation geometry.

To handle this representation drift, Dysco progressively boosts subspaces across rounds instead of overwriting them.
At round $t+1$, clients recompute local subspaces from the updated features $H_i^{t+1}$, and the server constructs new merged subspaces $A_i^{\mathrm{merge},t+1}$.
Each client freezes the previously learned adapter pair $(B_i^t, A_i^{\mathrm{merge},t})$ and adds a new LoRA block $(B_i^{t+1}, A_i^{\mathrm{merge},t+1})$, where $B_i^{t+1}$ is trained from scratch.
The frozen adapters preserve previously learned update directions, while the newly added block adapts to the current representation geometry.
The final model parameters take the form
\begin{equation}
    \label{eq:boosted_param}
    W = W_0 + \sum_{t=0}^{R-1} B_i^t\, A_i^{\mathrm{merge},t},
\end{equation}
where $R$ is the total number of communication rounds.
Following~\citep{OSRM}, we relax the orthonormality constraint of $A_i^{\mathrm{merge},t}$ during local training, allowing the subspace to be fine-tuned jointly with $B_i^t$ for improved performance.

\begin{algorithm}[t]
    \caption{\textbf{Dy}namic \textbf{S}ubspa\textbf{c}e B\textbf{o}osting \textbf{(Dysco)}}
    \label{alg:method}
      \begin{flushleft}
        \textbf{Input} a pre-trained model $f_0$, clients $[1,\dots,N]$,
        number of layers $L$, an aggregation algorithm $\mathcal{A}$, and a local training algorithm $\mathcal{T}$.
        \end{flushleft}
   \begin{algorithmic}[1]
        \While{not converge}
        \State \textit{/* Local subspace selection */}
        \For{$\text{client}=1,\cdots,N$}
            \State Compute the optimal subspace $A_i^*$ based on~\cref{eq:compute subspace}
            \State Send $A_i^*$ to the server
        \EndFor
        \State \textit{/* Server-side subspace merging */}
        \For{server}
            \For{$\text{client}=1,\dots,N$}
                \State Merge the subspaces based on~\cref{eq:merge_obj} as $A_i^{\mathrm{merge}}$
                \State Send $A_i^{\mathrm{merge}}$ back to client~$i$
            \EndFor
        \EndFor
        \State \textit{/* Multi-round subspace boosting */}
        \For{$\text{client}=1,\dots,N$}
            \State Initialize the subspace with $A_i^{\mathrm{merge}}$
            \State Train $B_i$ on local data as $\Tilde{B}_i = \mathcal{T}(B_i)$
            \State Send $\Tilde{B}_i$ to the server
        \EndFor
        \State \textit{/* Aggregating models */}
        \For{server}
            \State Aggregate models as $\Bar{B} = \mathcal{A}(\set{B_i}_{i=1}^N)$
            \State Send $\Bar{B}$ back to the clients
        \EndFor
        \EndWhile
   \end{algorithmic}
\end{algorithm}

\subsection{Optimization Comparison with FedAvg-LoRA}
\label{subsec:method-theory-summary}

The data-parameter interference of~\eqref{eq:interference} enters the standard nonconvex local-SGD convergence bound as an additive aggregation-error term.
For FedAvg-LoRA this term is bilinear in the average cross-client drifts of both LoRA factors $\Delta_A^{t+1}\Delta_B^{t+1}$, because the server averages both $A$ and $B$ across clients whose locally trained subspaces disagree under data heterogeneity.
Dysco fixes the right factor on a per-client basis to the server-merged $A_i^{\mathrm{merge},t+1}$, both during local training and at global application, which eliminates the $A$-drift contribution \emph{by construction} and collapses the interference to a linear function of the left-factor drift alone, scaling as $\rho_A^{t+1}\Delta_B^{t+1}$ with $\rho_A^{t+1}$ the maximum squared spectral norm of the merged right factors and $\rho_A^{t+1}=1$ for orthonormal merged subspaces.
This structural collapse yields a $\Theta(\Delta_A/\rho_A)$ tightening of the per-round stationarity bound and a $\Theta(K^2)$ widening advantage in the local-step count $K$, under standard data-heterogeneity conditions.
The full statement, assumptions, and proofs are deferred to~\cref{app:dysco_vs_fedavg}.

\section{Synthetic Experiments}
\label{sec:analysis}

\subsection{Experimental Setup}
\label{subsec:synthetic settings}
We construct a controlled federated scenario in which the subspace-allocation view can be tested directly.
There are $N$ clients, each with a vector-valued task: for client $i$ we draw a column-orthonormal subspace $U_i \in \mathbb{R}^{d_{\mathrm{model}} \times r}$ and generate input tokens $x = Z_i U_i^\top$ that lie entirely in $U_i$, with labels produced by a per-client target operator $W_i^\star = (1-\alpha)\,W_{\mathrm{shared}} + \alpha\,W_{\mathrm{perturb},i}$ at the maximum-heterogeneity setting $\alpha=1$.
Local training uses a Transformer encoder with LoRA adapters of rank $r$ attached to selected attention projections from $Q$, $K$, $V$, and all non-LoRA parameters are frozen.
We report the merged global model's average training loss and the corresponding top-1 test accuracy to measure the generation capability of the Transformer.
Full data-generation details, architecture specification, and training hyperparameters are deferred to~\cref{app:synthetic-setup}.

\subsection{Main Results}
\label{subsec:synthetic results}

\cref{fig:synthetic-depth} reports training loss and test accuracy for stacked Transformers of depth $L\in\{1,2,5\}$.
Dysco substantially outperforms FedAvg at every depth: its final training loss is $0.11$ against FedAvg's $0.57$ for $L=1$, $0.06$ against $0.51$ for $L=2$, and $0.05$ against $0.46$ for $L=5$, with each gap exceeding the seed standard deviation.
By round $20$, Dysco's training loss has already fallen below FedAvg's round-$100$ loss at every depth, showing an order-of-magnitude reduction in communication cost at matched performance.
The same pattern holds on test accuracy: Dysco reaches $0.53$, $0.59$, $0.44$ for $L=1,2,5$, respectively, while FedAvg remains at $0.15$, $0.18$, $0.24$, respectively.
By fixing the right LoRA factor at the data-aligned subspace, Dysco collapses the bilinear cross-client interference in FedAvg-LoRA to a linear function of the drift in $B$ alone, and this advantage persists as the network deepens.

\begin{figure}[t]
    \centering
    \includegraphics[width=.7\linewidth]{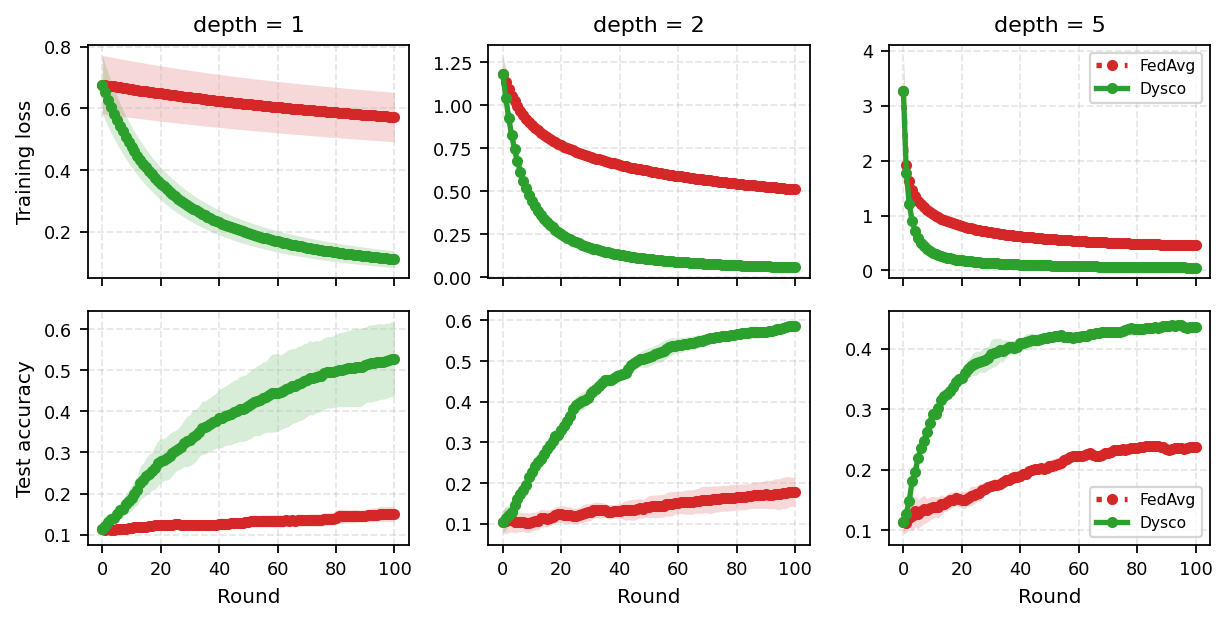}
    \caption{Training loss (top) and test accuracy (bottom) across Transformer depths $L\in\{1,2,5\}$. Dysco's subspace allocation persists as the network deepens, substantially lowering training loss and lifting test accuracy at every depth.}
    \label{fig:synthetic-depth}
\end{figure}

\subsection{Effect of Learnable LoRA Blocks}
\label{subsec:block ablation}

\cref{fig:synthetic-blocks} compares LoRA on $V$ alone with LoRA on $Q\!+\!K\!+\!V$; the underperforming $Q$-only and $K$-only ablations are reported in~\cref{app:blocks-qk}.
Restricting LoRA to $V$ alone recovers most of the gain of the full $Q\!+\!K\!+\!V$ configuration.
Dysco's final training loss is $0.08$ against $0.06$, and its test accuracy is $0.58$ against $0.59$, while FedAvg stays above $0.51$ loss and at most $0.18$ accuracy in either column.
The mechanism is that $V$ writes content into the residual stream, so a client-aligned LoRA basis lands one client's value updates in directions orthogonal to other clients' representations, whereas $Q$ and $K$ enter only through a softmax that is largely insensitive to the LoRA basis.
Aligning the subspace of the projection that \emph{writes} into the residual stream is therefore necessary and nearly sufficient for the Dysco lift in this controlled setting.

\begin{figure}[t]
    \centering
    \includegraphics[width=.5\linewidth]{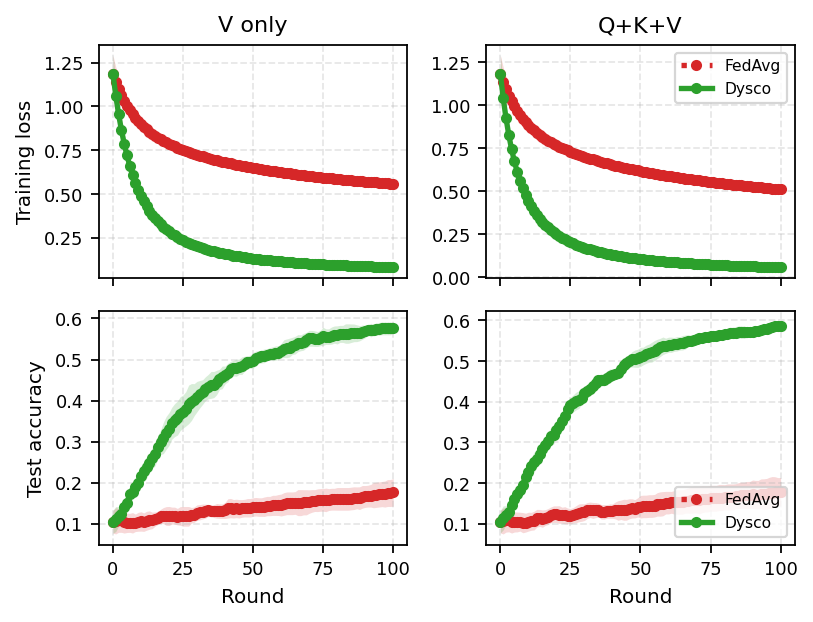}
    \caption{Training loss (top) and test accuracy (bottom) when LoRA targets only $V$ versus all three $Q\!+\!K\!+\!V$ projections. Aligning LoRA on the value projection alone captures nearly all of the $Q\!+\!K\!+\!V$ gain, because $V$ writes interfering content into the residual stream.}
    \label{fig:synthetic-blocks}
\end{figure}

\subsection{Scaling with Number of Clients}
\label{subsec:synthetic num clients}

\cref{fig:impact of client num} sweeps $N\in\{4,16,100\}$ at rank $r=4$, and the boundary case $N=64$ is deferred to~\cref{app:clients-N64}.
At $N=4$ where $N\!\cdot\!r=16$ is far below $d=128$, the strict orthogonal partition is feasible and Dysco drives the training loss to $0.05$ while FedAvg plateaus at $0.36$.
At $N=16$, FedAvg stagnates near its initial loss for the full $100$ rounds, whereas Dysco still cuts the loss by an order of magnitude to $0.10$.
At $N=100$, the strict partition is no longer feasible and Dysco partially degenerates toward FedAvg by construction; its final loss is $0.30$ against FedAvg's $0.55$.
Test accuracy follows the same trend.
Dysco yields $0.55$, $0.41$, $0.17$ for $N=4,16,100$, respectively, while FedAvg collapses to the random-chance floor of $0.10$ as soon as $N$ leaves the strict-orthogonal regime.
The data-aligned allocation therefore continues to extract benefit from the per-client random orthonormal fallback, since independent random orthonormal $U_i$ still suppress the cross-client bilinear interference term in expectation even without strict mutual orthogonality.

\begin{figure}[t]
    \centering
    \includegraphics[width=.7\linewidth]{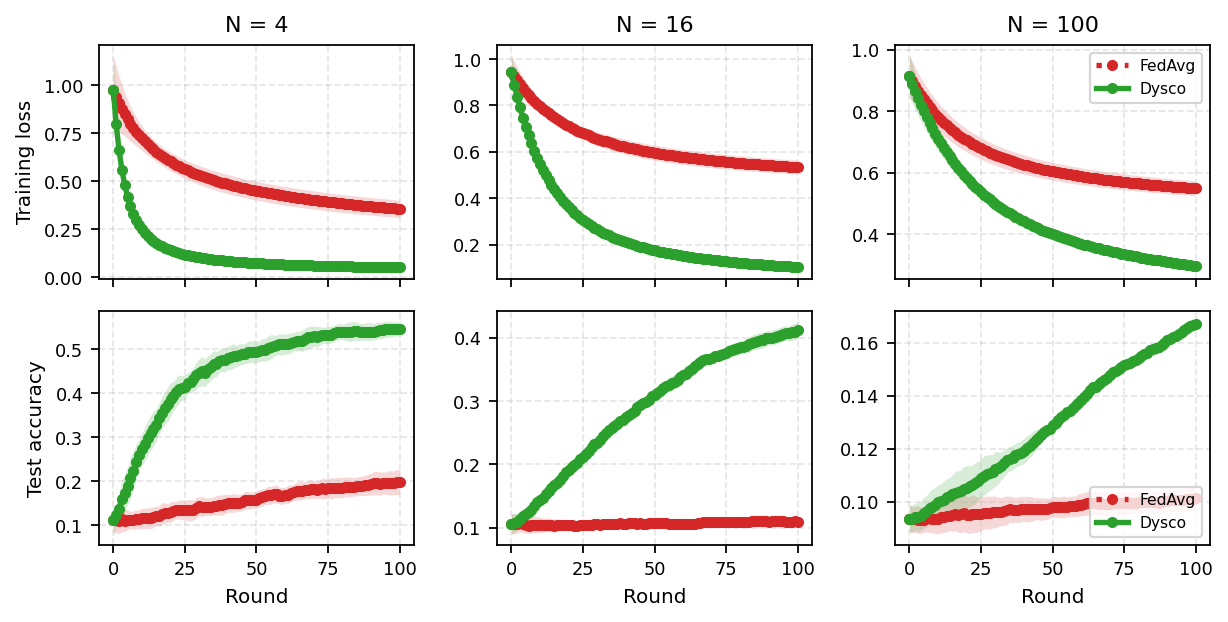}
    \caption{Training loss (top) and test accuracy (bottom) as the number of clients $N$ varies across $\{4,16,100\}$ at rank $r=4$. Dysco dominates FedAvg when $N\!\cdot\!r\le d$ permits a strict orthogonal partition, and retains a clear lead via random orthonormal fallback as $N$ grows.}
    \label{fig:impact of client num}
\end{figure}

\subsection{Effect of LoRA Rank}
\label{subsec:synthetic rank}

\cref{fig:impact of lora rank} sweeps the LoRA rank $r\in\{8,16,32\}$, and the small-capacity case $r=4$ is deferred to~\cref{app:rank-r4}.
Higher ranks amplify cross-client interference for FedAvg, whose loss grows from $0.51$ at $r=8$ to $0.76$ at $r=16$ and $1.21$ at $r=32$, while Dysco's loss is $0.06$ at $r=8$, $0.11$ at $r=16$, and $0.26$ at $r=32$, with the gap widest at $r=8$.
The test accuracy has a similar trend, where Dysco's accuracy is $0.59$ at both $r=8$ and $r=16$ and $0.53$ at $r=32$, while FedAvg's accuracy never exceeds $0.24$.
Higher-dimensional adapters are thus useful only when each client's update is allocated to a subspace insensitive to the others'.

\begin{figure}[t]
    \centering
    \includegraphics[width=.7\linewidth]{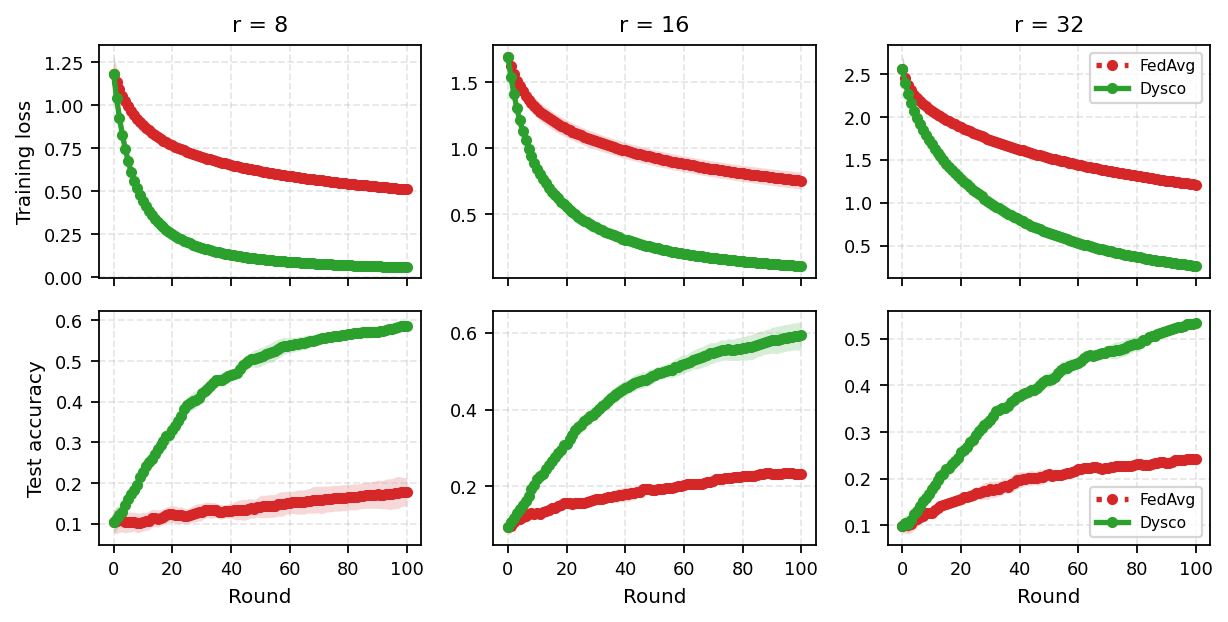}
    \caption{Training loss (top) and test accuracy (bottom) for LoRA ranks $r\in\{8,16,32\}$. Higher LoRA rank amplifies FedAvg's interference, while Dysco converts the extra capacity into stable accuracy at every rank.}
    \label{fig:impact of lora rank}
\end{figure}

\section{Real-World Experiments}
\label{sec:experiments}

\begin{table}[t]
\centering
\small
\caption{Performance (\%) of the global model on the MIMIC-IV-note dataset.
Dysco lifts the average accuracy of every tested FL algorithm, with gains ranging from $+0.5$ (Scaffold) to $+4.3$ (FedAvgM) points.}
\label{tab:mimic results}
\setlength{\tabcolsep}{5pt}
\begin{tabular}{l c ccccc}
\toprule
Algorithm & Dysco & LIVER & COAG & NEURO & WGHTLOSS & Avg. \\
\midrule

\multirow[c]{2}{*}{FedAvg}
& No & \textbf{86.8} & 79.0 & 83.2 & 74.3 & 80.8 \\
& Yes     & 86.6 & \textbf{80.7} & \textbf{85.2} & \textbf{77.5} & \textbf{82.5} \\

\midrule
\multirow[c]{2}{*}{FedAvgM}
& No & 86.6 & 79.6 & 80.1 & 73.4 & 79.9 \\
& Yes     & \textbf{87.6} & \textbf{82.1} & \textbf{86.8} & \textbf{80.2} & \textbf{84.2} \\

\midrule
\multirow[c]{2}{*}{FedProx}
& No & 83.5 & 76.2 & 81.2 & 72.2 & 78.3 \\
& Yes     & \textbf{86.1} & \textbf{79.2} & \textbf{84.7} & \textbf{77.6} & \textbf{81.9} \\

\midrule
\multirow[c]{2}{*}{Scaffold}
& No & 47.5 & \textbf{54.7} & \textbf{48.2} & 47.6 & 49.5 \\
& Yes     & \textbf{52.2} & 52.3 & 48.1 & 47.6 & \textbf{50.0} \\

\midrule
\multirow[c]{2}{*}{FedNova}
& No & 85.1 & 78.3 & 81.6 & 75.6 & 80.2 \\
& Yes     & \textbf{86.3} & \textbf{79.1} & \textbf{85.1} & \textbf{78.4} & \textbf{82.2} \\

\bottomrule
\end{tabular}
\end{table}

\begin{table}[t]
\centering
\caption{Average wall-clock time per round (seconds) across all settings for each FL method, comparing Baseline and Dysco. The Overhead (\%) column is the simple mean of per-setting percent overheads (c.f.~\cref{tab:overhead-detail}).
The Baseline (s) and Dysco (s) columns are independent across-setting time averages and therefore do not algebraically reproduce the percent overhead.}
\label{tab:overhead-avg}
\begin{tabular}{lccc}
\toprule
Method   & Baseline (s) & Dysco (s) & Overhead (\%) \\
\midrule
FedAvg   & 1523.8 & 1533.8 & +1.5 \\
FedProx  & 1587.8 & 1604.9 & +0.9 \\
Scaffold & 1537.2 & 1556.1 & +0.6 \\
\midrule
\textbf{Overall} & \textbf{1549.6} & \textbf{1564.9} & \textbf{+0.9} \\
\bottomrule
\end{tabular}
\end{table}

\subsection{Experimental Setup}
\label{subsec:exp-setup}

\paragraph{Datasets.}
We evaluate Dysco's per-task accuracy on the MIMIC-IV clinical-note benchmark in the main body and defer the corresponding per-task GLUE results to~\cref{sec:appdx-glue}.
MIMIC-IV-note~\citep{mimic-iv-note,physiobank} is a dataset consisting of de-identified free-text clinical notes for patients in the MIMIC-IV database~\citep{mimic-iv,mimiciv-3.1}, containing diverse note types, including discharge summaries, radiology reports, nursing notes, and physician progress notes, spanning multiple hospital departments and care settings.
This dataset provides a realistic and challenging benchmark for evaluating models on clinical language understanding tasks due to its heterogeneity, domain-specific terminology, and unstructured nature.
Following standard practice, we preprocess the notes by removing protected health information and normalizing text.
We select four diseases, including neurological disorders (NEURO), liver diseases (LIVER), coagulation deficiency (COAG), and weight loss (WGHTLOSS).
Each client is assigned a binary classification of a disease.

\paragraph{Models.}
For the MIMIC-IV-note dataset, we use Llama-3.2-1B~\citep{llama3}.
Llama-3.2-1B is a compact decoder-only large language model that yields strong performance per parameter.
We use Llama-3.2-1B to evaluate the behavior of our method on a modern decoder-only LLM backbone in a federated clinical-text setting.

\paragraph{Baselines.}
As Dysco is a plug-in method that enhances existing FL algorithms, we select five commonly used algorithms in FL and compare the performance w/ and w/o Dysco, including
FedAvg~\citep{fedavg}, FedAvgM~\citep{fedavgm}, Scaffold~\citep{scaffold}, FedProx~\citep{fedprox}, and FedNova~\citep{fednova}.

\paragraph{Configurations.}
During training, we use the AdamW optimizer~\citep{adamw} with a warmup ratio of $0.06$ and a linear learning rate schedule.
Following~\citep{lora}, LoRA is set with a rank of $r=8$, a scaling factor of $\alpha=16$, and is only applied to the query and value blocks.
The number of local epochs is $5$, and the total number of communication rounds is $30$.

\subsection{Main Results}
\label{subsec:main-results}

\paragraph{Training Loss Convergence.}
\cref{fig:training-loss-convergence} plots the average training loss versus communication round for both benchmarks under the standard configuration described in~\cref{subsec:exp-setup}.
On GLUE, Dysco's loss drops below FedAvg's after round~$3$ and remains lower for the rest of training, reaching $0.06$ at round~$30$ while FedAvg plateaus at $0.10$.
On MIMIC-IV-note, the gap opens earlier and persists for the full $30$ rounds, with Dysco's loss settling at $0.08$ against FedAvg's $0.11$.
The speed improvement on MIMIC-IV-note is less significant than that on GLUE, since clinical notes have more similar underlying semantics and the domain gap is smaller. 
It implies the benefit of Dysco under heavy data heterogeneity.
For brevity, in the following sections, we only show the results of MIMIC-IV-note in the main body and defer those of GLUE in~\cref{app:real-world-additional}

\begin{figure}[t]
    \centering
    \includegraphics[width=.7\linewidth]{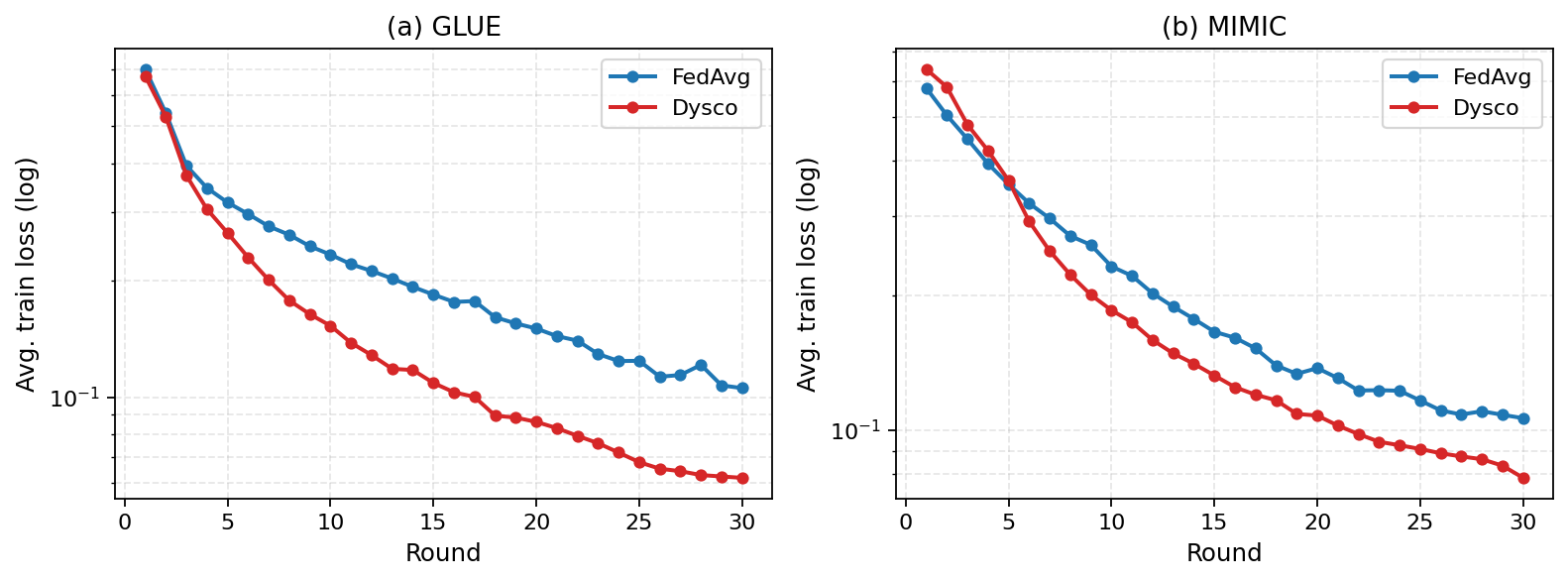}
    \caption{Average training loss vs.\ communication round on GLUE (a) and MIMIC-IV-note (b) under the standard configuration. Dysco's loss falls below FedAvg's within the first few rounds on both benchmarks and remains lower for the remainder of training.}
    \label{fig:training-loss-convergence}
\end{figure}

\paragraph{Final Accuracy.}
\cref{tab:mimic results} summarizes the classification accuracies (\%) on the MIMIC-IV-note dataset across the four clinical prediction tasks as well as the averaged performance over all tasks.
Dysco leads to consistent and substantial improvements in global model performance across nearly all tasks and FL settings.
When averaged over tasks, enabling Dysco improves accuracy for every baseline considered, with gains ranging from $0.5\%$ for Scaffold up to $4.3\%$ for FedAvgM.
This indicates that the benefits of Dysco are robust and not tied to a specific optimization strategy or aggregation rule.
At the task level, Dysco yields clear improvements for the majority of method-task combinations, with peak per-task gains of $3.0\%$ for FedProx on COAG and $6.7\%$ for FedAvgM on NEURO.
The contribution of each Dysco component is empirically validated in~\cref{app:ablation}.
The corresponding analysis on the GLUE benchmark is provided in~\cref{sec:appdx-glue}.

\paragraph{Comparison with State-of-the-Art Federated LoRA Methods.}
We compare Dysco against two recent federated LoRA methods, FFA-LoRA~\citep{ffa-lora} and FedSA-LoRA~\citep{fedsa-lora}, which address heterogeneity by freezing or structurally partitioning LoRA adapters.
\cref{tab:mimic-sota} reports the performance on the MIMIC-IV-note dataset.
FedAvg with Dysco outperforms both FFA-LoRA and FedSA-LoRA, demonstrating that dynamically resolving subspace interference is more effective than static structural constraints.

\begin{table}[t]
\centering
\small
\caption{Comparison with state-of-the-art federated LoRA methods on the MIMIC-IV-note dataset. Performance is reported in \%.}
\label{tab:mimic-sota}
\setlength{\tabcolsep}{5pt}
\begin{tabular}{l ccccc}
\toprule
Method & LIVER & COAG & NEURO & WGHTLOSS & Avg. \\
\midrule
FedAvg              & \textbf{86.8} & 79.0 & 83.2 & 74.3 & 80.8 \\
FedAvg + Dysco      & 86.6 & \textbf{80.7} & \textbf{85.2} & \textbf{77.5} & \textbf{82.5} \\
\midrule
FFA-LoRA            & 82.7 & 75.7 & 81.5 & 73.3 & 78.3 \\
FedSA-LoRA          & 84.5 & 77.5 & 83.8 & 77.3 & 80.8 \\
\bottomrule
\end{tabular}
\end{table}

\paragraph{Computational Overhead.}
We measure the wall-clock time per federated round for three representative FL methods, including FedAvg, FedProx, and Scaffold, across five settings spanning MIMIC-IV-note with Llama-3.2-1B and the GLUE configurations evaluated in~\cref{sec:appdx-glue}, each running a single outer round.
\cref{tab:overhead-avg} reports the average wall-clock time per round for each method.
Across all three methods and five settings, Dysco adds only $0.9\%$ wall-clock overhead on average, with per-method averages ranging from $0.6\%$ for Scaffold to $1.5\%$ for FedAvg.
The subspace computation and cross-client aggregation introduced by Dysco are dominated by the LoRA training cost, making the additional overhead negligible in practice.
A per-setting breakdown is provided in~\cref{tab:overhead-detail} in the appendix.

\section{Related Work}
\label{sec:related work}

\paragraph{Handling Data Heterogeneity in FL.}
Data heterogeneity is a central challenge in federated learning because local optimization on client-specific distributions can induce client drift and degrade the aggregated global model~\citep{moon,fedagrac}.
A large body of work improves robustness to non-IID data by modifying the federated optimization procedure.
FedProx constrains local updates with a proximal term~\citep{fedprox}, Scaffold uses control variates to correct biased local directions~\citep{scaffold}, and FedNova normalizes client updates to reduce objective inconsistency caused by heterogeneous local training~\citep{fednova}.
Other approaches use auxiliary data, generated data, or distillation to reduce distribution mismatch~\citep{fedaux,fedftg,fedgen,feddf}, while personalized FL learns client-specific models or adaptations for heterogeneous populations~\citep{IFCA,ditto,fedamp,perfedavg}.
These methods primarily address heterogeneity through optimization, calibration, auxiliary information, or personalization.
Our work is complementary. We retain a single shared global model, but control the LoRA update subspaces in which client updates are learned, thereby reducing cross-client interference before aggregation.

\paragraph{Subspace and Null-Space Methods.}
Subspace constraints have been widely used to reduce interference among tasks, updates, or model edits.
In continual learning, gradient projection methods restrict new updates to directions that minimally affect previous tasks~\citep{OGD}.
In multi-task learning, gradient surgery removes conflicting components across task gradients~\citep{PCGrad}.
Related ideas also appear in parameter-efficient adaptation, where Null-LoRA constrains LoRA updates to the null space of the pre-trained model~\citep{Null-LoRA}, as well as in knowledge editing and model merging, where null-space or orthogonal-subspace constraints are used to preserve existing behavior while incorporating new information~\citep{NSPO,alphaedit,alphasteer,sonoedit,OSRM,Nufilt,mingle}.
Most of these methods compute the relevant subspace from centralized information or use a fixed subspace throughout the procedure.
In contrast, Dysco targets the federated setting where raw data and intermediate activations cannot be pooled, and where the activation geometry changes across communication rounds.
It therefore performs subspace selection locally, merges only transmitted bases on the server, and dynamically refreshes the allocation through multi-round boosting.

\paragraph{Federated LoRA.}
LoRA has become a common tool for parameter-efficient fine-tuning of large language models by learning low-rank adapters while freezing the backbone~\citep{hu2022lora}.
Recent work extends LoRA to FL to reduce communication and computation in distributed fine-tuning.
pFedLoRA and FDLoRA introduce personalized LoRA designs for model or data heterogeneity~\citep{yi2023pfedlora,qi2024fdlora}.
FFA-LoRA and FedSA-LoRA improve federated adapter aggregation by freezing or selectively sharing LoRA factors~\citep{sun2024improving,guo2025selective}.
Other methods, including HetLoRA, FlexLoRA, FLoRA, and LoRA-A$^2$, support heterogeneous ranks, tasks, or client resources through specialized aggregation and rank-allocation strategies~\citep{cho2024heterogeneous,bai2024federated,wang2024flora,koo2025towards}.
These methods focus on how LoRA adapters should be shared, personalized, or aggregated under federated constraints.
Dysco addresses a different and complementary issue: the data-parameter interference that arises when an aggregated adapter from one client acts on another client's representations.
By allocating, merging, and boosting client-specific update subspaces, Dysco can be plugged into standard FL algorithms to improve LoRA-based federated fine-tuning under heterogeneity.

\section{Conclusion}
\label{sec:conclusion}

We studied LoRA-based federated fine-tuning under client heterogeneity and traced unstable aggregation to data-parameter interference, where client updates collide when projected into a shared subspace. This motivates a subspace-allocation view, in which client updates should occupy subspaces that are locally expressive yet insensitive to other clients' representations. We proposed Dysco, a plug-in method that instantiates this view through local activation-insensitive subspace selection, closed-form server-side subspace merging, and multi-round subspace boosting that re-allocates as representations evolve. Empirically, Dysco shrinks the final-round training loss by roughly $5$--$9\times$ over FedAvg-style baselines on controlled synthetic federated tasks, and on MIMIC-IV clinical-note classification it improves baselines with average gains up to $+4.3$ points while adding negligible computational overhead.

\paragraph{Limitations.}
Dysco assumes all clients share the same base model architecture, targeting data rather than model heterogeneity; handling heterogeneous architectures would require extra mechanisms such as distillation, module-wise merging, or architecture-aware client clustering. Dysco also relies on synchronous aggregation, since the server must collect all client subspace bases before merging, making it unsuitable for asynchronous settings.

\section*{Acknowledgment}

This material is based in part upon work supported by the National Science Foundation under Grant IIS-2212174, National Institute of Aging (NIA) R01AG072449, and National Institute of General Medical Sciences (NIGMS) R01GM145700.
Fei Wang would also like to acknowledge the support from NSF 2212175, NIH RF1AG072449, RF1AG084178, R01AG080991, R01AG080624, R01AG076448, R01AG076234, OT2OD038083, R01AG072449, R01NS140142 and R01NS148533.

\bibliographystyle{apalike}
\IfFileExists{/Users/jyhong/Code/template/latex-utils/zotero_bib.bib}{
	\bibliography{/Users/jyhong/Code/template/latex-utils/zotero_bib}  %
}
{
	\bibliography{references}  %
}

\clearpage

\appendix
\addtocontents{toc}{\protect\etocsetlevel{section}{1}}
\addtocontents{toc}{\protect\etocsetlevel{subsection}{6}}
\addtocontents{toc}{\protect\etocsetlevel{subsubsection}{6}}
\addtocontents{toc}{\protect\etocsetlevel{paragraph}{6}}
\etocsettocstyle{\section*{Appendix Contents}}{\bigskip}
\etocstandardlines
\begingroup
\etocsetlevel{section}{6}
\etocsetlevel{subsection}{6}
\etocsetlevel{subsubsection}{6}
\etocsetlevel{paragraph}{6}
\setcounter{tocdepth}{1}
\tableofcontents
\endgroup

\section{Broader Impact}
\label{sec:broader-impact}

Dysco aims to improve the reliability of federated adaptation for large language models in privacy-sensitive domains by reducing instability induced by heterogeneity. 
Promising directions include: (i) tighter privacy analyses for communicating subspace bases, (ii) principled subspace scheduling guided by interference diagnostics, and (iii) extending dynamic subspace control to other PEFT forms and to personalized or clustered FL settings. 
Overall, dynamic subspace control provides a practical and theoretically grounded tool for making LoRA-based FL more stable under real-world non-IID data.

\section{Comparison between Dysco and FedAvg Under Data Heterogeneity}
\label{app:dysco_vs_fedavg}

This appendix provides a direct optimization-side comparison between Dysco and FedAvg-LoRA.
We adopt the aggregation-error framework of \citet{fedrotlora}, embedding the data-parameter interference as an explicit term $\|\bar E^{t+1}\|_F^2$ in a standard local-SGD convergence bound, and show that Dysco's server-fixed per-client right factor structurally collapses the bilinear cross-client interference of FedAvg-LoRA to a linear one, yielding a $\Theta(\Delta_A/\rho_A)$ tightening of the iterate's optimization error across two convergence regimes, including per-round nonconvex stationarity and $K$-step local SGD.

\subsection{Setup and Notation}
\label{app:dvf_setup}

We focus on a single LoRA-adapted layer of a fixed network; the statements apply per layer and per LoRA block and can be summed over layers.
For each of $N$ clients $i\in\{1,\dots,N\}$, let $f_i:\R^{d\times d}\to\R$ denote client~$i$'s population objective and $f\triangleq\tfrac1N\sum_i f_i$ the global objective.
We write $W_0$ for the pretrained initialization, $\Delta W = W - W_0$ for the update, and $W^\star\in\argmin_W f(W)$, $f^\star=f(W^\star)$ for the (unconstrained) global minimum.
The LoRA factors at client $i$ and round $t$ are denoted $(B_i^t,A_i^t)\in\R^{d\times r}\times\R^{r\times d}$, and the corresponding adapter contribution is $B_i^t A_i^t$.
In the theoretical analysis, we consider a more general aggregation form, written as $\bar B^t\triangleq\tfrac1N\sum_i B_i^t$ and $\bar A^t\triangleq\tfrac1N\sum_i A_i^t$.
Throughout we assume $f$ is $\beta$-smooth in the Frobenius norm:

\begin{assumption}[$\beta$-smoothness]
\label{assump:smooth}
For all $W$ and $U$, $f(W+U)\le f(W) + \langle \nabla f(W), U\rangle + \frac{\beta}{2}\|U\|_F^2$.
\end{assumption}

\subsection{Data-Parameter Interference in the Convergence Bound}
\label{app:dvf_aggerror}

We develop an optimization-side convergence analysis in four steps:
(i) define the {\em data-parameter interference} as an explicit per-round quantity,
(ii) prove Dysco's per-client server-fixed right factor structurally reduces it,
(iii) embed this term inside a nonconvex stationarity bound, and
(iv) quantify the resulting tightening.

The {\em data-parameter interference}
measures the gap between the model and algorithm actually applies to client data and the idealized aggregate that averages locally-computed rank-$r$ updates.
We
insert this gap as an additive perturbation $E^{t+1}=W^{t+1}-W^{t+1}_{\mathrm{ideal}}$ inside the standard nonconvex SGD descent inequality.
Our contribution is the structural observation that Dysco collapses $E^{t+1}$ from a bilinear cross-client interference (drift in both LoRA factors) into a linear one (drift in the left factor $B$ only).
We quantify the resulting tightening in two regimes: per-round nonconvex stationarity (\cref{cor:tighter_conv}) and $K$-step local SGD (\cref{thm:agg_K}).

\subsubsection{Regime I: Per-round Nonconvex Stationarity}
\label{app:dvf_regimeI}

\paragraph{Additional optimization assumptions.}
We add two mild assumptions that are standard in federated LoRA convergence analysis.

\begin{assumption}[Bounded stochastic LoRA-factor gradients]
\label{assump:bdd_grad}
For each client $i$, round $t$, parameter block $\theta\in\{A,B,W\}$, and sample $\xi_i$, the stochastic gradient $\widehat\nabla_\theta f_i(W^t,\xi_i)$ is unbiased ($\E[\widehat\nabla_\theta f_i(W^t,\xi_i)]=\nabla_\theta f_i(W^t)$) and uniformly bounded almost surely:
$\|\widehat\nabla_\theta f_i(W^t,\xi_i)\|_F\le G_\theta$, which in particular implies $\|\nabla_\theta f_i(W^t)\|_F\le G_\theta$.
\end{assumption}

\begin{assumption}[Bounded LoRA-factor magnitudes]
\label{assump:bdd_lora}
For every client $i$ and round $t$, the local LoRA factors satisfy $\|B_i^t\|_F\le\tau_B$ and $\|A_i^t\|_F\le\tau_A$.
\end{assumption}

These mirror Assumptions 4.2 and 4.3 of the FedRot-LoRA analysis and are non-restrictive in fine-tuning regimes.

\paragraph{Defining the data-parameter interference.}
At round $t+1$, the per-client {\em data-parameter interference} is the gap between the algorithm's actual update at client $i$ and the {\em ideal} update obtained by averaging client-level rank-$r$ updates rather than averaging factors:
\begin{equation}
\label{eq:aggerr_def}
E_i^{t+1,\mathrm{ALG}}
\;\triangleq\;
W_i^{t+1,\mathrm{ALG}} \;-\; W_i^{t+1,\mathrm{ideal,ALG}},
\end{equation}
where $W_i^{t+1,\mathrm{ALG}}$ is the model the algorithm applies to client $i$'s data and $W_i^{t+1,\mathrm{ideal,ALG}}$ is the corresponding ideal aggregate of local update matrices.
Concretely:

\textbf{FedAvg-LoRA.}
The aggregated global model at every client is $W^{t+1,\mathrm{FA}}=W_0+\bar B^{t+1}\bar A^{t+1}$ with $\bar B^{t+1}=\tfrac1N\sum_i B_i^{t+1}$, $\bar A^{t+1}=\tfrac1N\sum_i A_i^{t+1}$.
The ideal aggregate is $W^{t+1,\mathrm{ideal,FA}}=W_0+\tfrac1N\sum_i B_i^{t+1}A_i^{t+1}$.
The polarization identity $\bar X\bar Y-\tfrac{1}{N}\sum_i X_iY_i=-\tfrac{1}{2N^2}\sum_{i,j}(X_i-X_j)(Y_i-Y_j)$ then yields
\begin{equation}
\label{eq:E_fa}
E^{t+1,\mathrm{FA}}
\;=\;
\bar B^{t+1}\bar A^{t+1}-\tfrac1N\sum_{i=1}^N B_i^{t+1}A_i^{t+1}
\;=\;
-\frac{1}{2N^2}\sum_{i,j=1}^N (B_i^{t+1}-B_j^{t+1})(A_i^{t+1}-A_j^{t+1}).
\end{equation}
This is the precise data-parameter interference that appears in~\eqref{eq:interference} of \cref{subsec:interference_nullspace}: it is bilinear in the cross-client drifts of {\em both} LoRA factors.

\textbf{Dysco.}
At round $t+1$, Dysco fixes the right factor on a per-client basis to the server-merged $A_i^{\mathrm{merge},t+1}$ (computed {\em before} local training and used identically in both local training and global application that round), so the model {\em applied} to client $i$'s data is $W_i^{t+1,\mathrm{Dy}}=W_0+\bar B^{t+1}A_i^{\mathrm{merge},t+1}$ with $\bar B^{t+1}=\tfrac1N\sum_j B_j^{t+1}$; the {\em ideal} update at client $i$ uses its own locally trained $B_i^{t+1}$ under the {\em same} $A_i^{\mathrm{merge},t+1}$, giving $W_i^{t+1,\mathrm{ideal,Dy}}=W_0+B_i^{t+1}A_i^{\mathrm{merge},t+1}$.
(The multi-round boosting in~\eqref{eq:boosted_param} adds previously frozen blocks to this round's contribution; the analysis below treats one boosting block per round, which is the granularity at which the interference is incurred.)
Hence
\begin{equation}
\label{eq:E_dy}
E_i^{t+1,\mathrm{Dy}}
\;=\;
\big(\bar B^{t+1}-B_i^{t+1}\big)\,A_i^{\mathrm{merge},t+1}.
\end{equation}
The structural identity here is sharper than ``smaller'': Dysco enforces $A_i^{\mathrm{train}}\equiv A_i^{\mathrm{agg}}\equiv A_i^{\mathrm{merge},t+1}$, so the right-factor drift across clients is {\em identically zero by construction} (not just upper-bounded).
The matrix-product inequality on the surviving residual then makes \eqref{eq:E_dy} {\em linear} in the cross-client drift of the left factor alone.

\paragraph{Cross-client drift quantities.}
The interferences in~\eqref{eq:E_fa}--\eqref{eq:E_dy} are naturally controlled by the post-local-training {\em client-drift} of the LoRA factors:
\begin{equation}
\label{eq:drift_def}
\Delta_B^{t+1}\triangleq\frac{1}{N^2}\sum_{i,j=1}^N\|B_i^{t+1}-B_j^{t+1}\|_F^2,
\qquad
\Delta_A^{t+1}\triangleq\frac{1}{N^2}\sum_{i,j=1}^N\|A_i^{t+1}-A_j^{t+1}\|_F^2.
\end{equation}
$\Delta_A^{t+1}$ and $\Delta_B^{t+1}$ are data-derived, algorithm-agnostic quantities that measure cross-client divergence of the locally trained factors; they depend on the underlying data heterogeneity but not on which aggregation rule is applied.

\paragraph{Dysco's structural reduction of the interference.}
The following lemma is the central inequality of this section: Dysco's interference is upper-bounded by a linear function of $B$-drift alone, while FedAvg-LoRA's is bounded only by the bilinear product $\Delta_A^{t+1}\Delta_B^{t+1}$.

\begin{lemma}[Interference comparison: FedAvg-LoRA vs.\ Dysco]
\label{lem:err_compare}
For every round $t+1$:
\begin{enumerate}[leftmargin=*,label={\rm(\arabic*)}]
\item ({\em FedAvg-LoRA: bilinear in $\Delta_A\Delta_B$.})
$\|E^{t+1,\mathrm{FA}}\|_F^2 \;\le\; \tfrac{1}{4}\,\Delta_A^{t+1}\,\Delta_B^{t+1}$.
\item ({\em Dysco: linear in $\Delta_B$, $A$-drift eliminated.})
$\tfrac{1}{N}\sum_{i=1}^N\|E_i^{t+1,\mathrm{Dy}}\|_F^2 \;\le\; \rho_A^{t+1}\cdot\tfrac{1}{2}\Delta_B^{t+1}$,
where $\rho_A^{t+1}\triangleq\max_i \|A_i^{\mathrm{merge},t+1}\|_2^2$.
By Jensen's inequality, $\|\bar E^{t+1,\mathrm{Dy}}\|_F^2 \le \rho_A^{t+1}\cdot\tfrac{1}{2}\Delta_B^{t+1}$.
\end{enumerate}
\end{lemma}

\begin{proof}
(1) Apply the triangle inequality and then Cauchy--Schwarz to~\eqref{eq:E_fa}:
\begin{align*}
\|E^{t+1,\mathrm{FA}}\|_F
&\le \tfrac{1}{2N^2}\sum_{i,j}\|B_i-B_j\|_F\,\|A_i-A_j\|_F\\
&\le \tfrac{1}{2N^2}\Big(\textstyle\sum_{i,j}\|B_i-B_j\|_F^2\Big)^{1/2}\Big(\textstyle\sum_{i,j}\|A_i-A_j\|_F^2\Big)^{1/2}
\;=\;\tfrac{1}{2}\sqrt{\Delta_B^{t+1}\Delta_A^{t+1}},
\end{align*}
using the identity $\sum_{i,j}\|X_i-X_j\|_F^2 = N^2\Delta_X^{t+1}$.
Squaring yields the stated bound.

(2) For each client $i$, the matrix-norm inequality $\|YZ\|_F\le\|Y\|_F\|Z\|_2$ (a consequence of submultiplicativity) gives
\[
\|E_i^{t+1,\mathrm{Dy}}\|_F^2
=\|(\bar B^{t+1}-B_i^{t+1})A_i^{\mathrm{merge},t+1}\|_F^2
\le \|\bar B^{t+1}-B_i^{t+1}\|_F^2\cdot\|A_i^{\mathrm{merge},t+1}\|_2^2.
\]
Averaging over $i$, using $\rho_A^{t+1}=\max_i\|A_i^{\mathrm{merge},t+1}\|_2^2$, and applying the variance-of-sums identity $\tfrac{1}{N^2}\sum_{i,j}\|B_i-B_j\|_F^2=\tfrac{2}{N}\sum_i\|B_i-\bar B\|_F^2$ gives $\rho_A^{t+1}\cdot\tfrac12\Delta_B^{t+1}$.
The Jensen step is standard.
\end{proof}

\begin{remark}[Orthonormality and the $\rho_A$ constant]
\label{rem:rhoA}
The Dysco merge step computes $A_i^{\mathrm{merge}}$ as orthonormal ($A_i^{\mathrm{merge}}A_i^{\mathrm{merge}\top}=I_r$).
Since the same $A_i^{\mathrm{merge}}$ governs both local training and global application within the round (cf.\ the Dysco paragraph above), the $A_i^{\mathrm{merge}}$ appearing inside~\eqref{eq:E_dy} is always the merge-time orthonormal frame, so $\rho_A^{t+1}=1$ {\em exactly} in the analysis.
As noted in \cref{subsec:dynamic boosting} (the paragraph following~\eqref{eq:boosted_param}), the practical implementation relaxes the orthonormality constraint during local training, following \citep{OSRM}; this relaxation only affects the $A$ {\em used at training}, perturbing $\rho_A^{t+1}\le(1+\varepsilon_{\mathrm{ortho}})^2$ for a small $\varepsilon_{\mathrm{ortho}}$ controlled by the local-optimizer step size.
Under that perturbation,~(2) becomes $(1+\varepsilon_{\mathrm{ortho}})^2\cdot\tfrac12\Delta_B^{t+1}$, leaving the $\Theta(\Delta_A)$ tightening of~\cref{cor:tighter_conv} unchanged in order of magnitude.
\end{remark}

\paragraph{Embedding the interference in a convergence bound.}
With Dysco's structural reduction (\cref{lem:err_compare}) in hand, we now insert $\|\bar E^{t+1}\|_F^2$ into a standard nonconvex stationarity bound; specializing the bound via \cref{lem:err_compare} then yields algorithm-specific rates.

\begin{theorem}[Convergence with explicit interference term]
\label{thm:agg_conv}
\begin{sloppypar}
Suppose \cref{assump:smooth} ($\beta$-smoothness), \cref{assump:bdd_grad} (bounded stochastic gradients), and \cref{assump:bdd_lora} (bounded LoRA norms) all hold.
Then any algorithm with per-client interferences $\{E_i^{t+1}\}_{i=1}^N$ produces a global model sequence $\{W^t\}_{t=0}^{T-1}$ satisfying
\end{sloppypar}
\begin{equation}
\label{eq:agg_conv}
\min_{0\le t<T} \E\!\left[\|\nabla_A f(W^t)\|_F^2 + \|\nabla_B f(W^t)\|_F^2\right]
\;\le\;
\frac{f(W^0)-f^\star}{\eta T}
\;+\;
\frac{3\beta\eta^2+1}{2T\eta}\,
\sum_{t=0}^{T-1}\E\!\left[\frac{\|\bar E^{t+1}\|_F^2}{\eta^2}\right]
\;+\;\mathcal{O}(\eta),
\end{equation}
where $\bar E^{t+1}\triangleq\tfrac1N\sum_{i=1}^N E_i^{t+1}$ is the average per-client interference and the $\mathcal{O}(\eta)$ term collects bounded contributions of the form $\eta(G_A^2+G_B^2)$ that do not depend on $\bar E^{t+1}$. For Dysco, each $E_i^{t+1}$ corresponds to one boosting block per round.
\end{theorem}

\begin{proof}[Proof sketch]
The argument follows the standard nonconvex local-SGD descent inequality with an inserted {\em ideal} iterate; the same structure appears in \citet[App.~A.1, Eq.~(40)]{fedrotlora} for FedAvg-LoRA's $E^{t+1,\mathrm{FA}}$.
Their derivation uses only that $E^{t+1}$ is an additive perturbation $W^{t+1}-W^{t+1}_{\mathrm{ideal}}$ (not its FedAvg-specific form), so it carries over verbatim to any algorithm by substituting that algorithm's $E^{t+1}$; we record the key steps for $\bar E^{t+1}$.
Decompose one round as $W^t\to W^{t+1,\mathrm{ideal}}\to W^{t+1}$. Three applications of $\beta$-smoothness---two for the legs $W^t\to W^{t+1,\mathrm{ideal}}$ and $W^{t+1,\mathrm{ideal}}\to W^{t+1}$, plus one absorbing the inner product $\langle\nabla f(W^{t+1,\mathrm{ideal}}),\bar E^{t+1}\rangle$ via Young's inequality with parameter $\eta^2$ (i.e.\ $\langle x,y\rangle\le\tfrac{1}{2\eta^2}\|x\|_F^2+\tfrac{\eta^2}{2}\|y\|_F^2$)---contribute coefficients $\beta/2$ each on $\|\bar E^{t+1}\|_F^2$ plus a single $1/(2\eta^2)$ from Young's, while the corresponding $(\eta^2/2)\|\nabla f\|_F^2$ Young term is absorbed into the descent $\|\nabla f\|_F^2$ contribution.
Tracking these contributions yields the coefficient $3\beta/2+1/(2\eta^2)=(3\beta\eta^2+1)/(2\eta^2)$ on $\|\bar E^{t+1}\|_F^2$; folding in the LoRA stochastic gradient bound (\cref{assump:bdd_grad}), the per-round descent reads
$\E[f(W^{t+1})]\le\E[f(W^t)]-\eta\,\E[\|\nabla_A f(W^t)\|_F^2+\|\nabla_B f(W^t)\|_F^2]+\frac{3\beta\eta^2+1}{2\eta}\E[\|\bar E^{t+1}\|_F^2/\eta^2]+\eta^2(G_A^2+G_B^2)$.
Telescoping over $t=0,\dots,T-1$ and dividing by $\eta T$ gives~\eqref{eq:agg_conv}.
\end{proof}

The bound~\eqref{eq:agg_conv} decomposes the stationarity rate into (i)~the initial-gap term, (ii)~the cumulative-interference term, and (iii)~a learning-rate residual.
Any algorithm that decreases $\|\bar E^{t+1}\|_F^2$ obtains a strictly tighter stationarity bound at the same learning rate---which, by \cref{lem:err_compare}, is precisely what Dysco achieves.

\begin{corollary}[Tightened convergence bound under data heterogeneity]
\label{cor:tighter_conv}
Substituting \cref{lem:err_compare} into \cref{thm:agg_conv} produces two algorithm-specific upper bounds on the same stationarity quantity $\min_t\E[\|\nabla_A f\|_F^2+\|\nabla_B f\|_F^2]$:
\begin{align*}
\text{FedAvg-LoRA:}\quad &
\frac{f(W^0)-f^\star}{\eta T}+\frac{3\beta\eta^2+1}{2T\eta}\,\sum_t\E\!\left[\frac{\Delta_A^{t+1}\Delta_B^{t+1}}{4\eta^2}\right]+\mathcal{O}(\eta), \\
\text{Dysco:}\quad &
\frac{f(W^0)-f^\star}{\eta T}+\frac{3\beta\eta^2+1}{2T\eta}\,\sum_t\E\!\left[\frac{\rho_A^{t+1}\Delta_B^{t+1}}{2\eta^2}\right]+\mathcal{O}(\eta).
\end{align*}
Dysco's per-round interference contribution is at most $\big(2\rho_A^{t+1}/\Delta_A^{t+1}\big)$ times FedAvg-LoRA's.
Whenever $\Delta_A^{t+1}\ge 2\rho_A^{t+1}$---a regime easily reached under data heterogeneity, since $\|A_i\|_F=\sqrt r$ and near-orthogonal $A$-factors give $\Delta_A^{t+1}=\Theta(r)$ while $\rho_A^{t+1}=1+\mathcal{O}(\varepsilon_{\mathrm{ortho}})$ (cf.\ \cref{rem:rhoA})---Dysco's bound is strictly tighter.
At the maximally heterogeneous operating point $\Delta_A^{t+1}=\Theta(r)$, the multiplicative tightening of the interference term scales as $\Theta(r)$.
\end{corollary}

\paragraph{Interpretation.}
By fixing $A_i^{\mathrm{merge}}$ both during local training and global application, Dysco eliminates the $A$-factor cross-client drift from the interference, making it linear in $B$-drift rather than bilinear in the product of $B$- and $A$-drifts; \cref{cor:tighter_conv} translates this structural reduction into a $\Theta(\Delta_A/\rho_A)$ tightening of the stationarity rate.

\subsubsection{Regime II: Multi-step Local SGD with $K$ Inner Steps}
\label{app:dvf_Kregime}

The per-round bound of \cref{thm:agg_conv} treats one local SGD step per round.
Practical federated LoRA fine-tuning uses $K\gg 1$ local steps to amortize communication; we extend the analysis to this regime and show that Dysco's structural advantage {\em strengthens} with $K$.

\paragraph{Multi-step drift accumulation.}
Each client performs $K$ inner SGD steps per round starting from the synchronized state $(\bar B^t,\bar A^t)$, so the trajectory is $B_i^{t,k+1}=B_i^{t,k}-\eta\,\widehat\nabla_B f_i(W_i^{t,k},\xi_i^{t,k})$, $k=0,\dots,K-1$, and similarly for $A_i^{t,k}$.
Under \cref{assump:bdd_grad,assump:bdd_lora}, applying the triangle inequality and the a.s.\ stochastic-gradient bound across the $K$ inner steps gives
\begin{equation}
\label{eq:Kdrift_bound}
\big\|B_i^{t,K}-B_j^{t,K}\big\|_F \;\le\; 2\eta K\,G_B,
\qquad
\big\|A_i^{t,K}-A_j^{t,K}\big\|_F \;\le\; 2\eta K\,G_A,
\end{equation}
so the cross-client drifts at round-end satisfy $\Delta_B^{t+1}\le 4\eta^2 K^2 G_B^2$ and $\Delta_A^{t+1}\le 4\eta^2 K^2 G_A^2$.

\begin{theorem}[$K$-step interference comparison]
\label{thm:agg_K}
Under \cref{assump:smooth,assump:bdd_grad,assump:bdd_lora}, with $K$ inner SGD steps and step size $\eta$ in each round, the per-round interference bounds of \cref{lem:err_compare} sharpen to
\begin{align*}
\|E^{t+1,\mathrm{FA}}\|_F^2
&\;\le\; \tfrac{1}{4}\Delta_A^{t+1}\Delta_B^{t+1} \;\le\; 4\eta^4 K^4 G_A^2 G_B^2, \\
\tfrac{1}{N}\sum_i\|E_i^{t+1,\mathrm{Dy}}\|_F^2
&\;\le\; \rho_A^{t+1}\cdot\tfrac{1}{2}\Delta_B^{t+1} \;\le\; 2\rho_A^{t+1}\eta^2 K^2 G_B^2.
\end{align*}
The worst-case interference ratio satisfies
$\|E^{t+1,\mathrm{FA}}\|_F^2 \big/ \tfrac{1}{N}\sum_i\|E_i^{t+1,\mathrm{Dy}}\|_F^2 \;\le\; 2\eta^2 K^2 G_A^2/\rho_A^{t+1}$, so Dysco's advantage grows as $\Theta(K^2)$ in the local-iteration count.
\end{theorem}

\begin{proof}
The bounds on $\|E^{t+1,\mathrm{FA}}\|_F^2$ and $\tfrac{1}{N}\sum_i\|E_i^{t+1,\mathrm{Dy}}\|_F^2$ are obtained by plugging the multi-step drift bound~\eqref{eq:Kdrift_bound} into \cref{lem:err_compare}.
The ratio is a direct computation: $\tfrac{1}{4}\Delta_A\Delta_B / (\rho_A\Delta_B/2)= \Delta_A/(2\rho_A)\le 4\eta^2 K^2 G_A^2 / (2\rho_A) = 2\eta^2 K^2 G_A^2/\rho_A^{t+1}$.
\end{proof}

\paragraph{Interpretation.}
Increasing $K$ widens Dysco's relative advantage as $\Theta(K^2)$, consistent with the empirical trend observed in~\cref{fig:synthetic-depth} where Dysco's accuracy gap over FedAvg compounds as deeper Transformers absorb more effective local update structure within each round.

\subsection{Summary of the Comparison}
\label{app:dvf_summary}

Embedding the data-parameter interference $\|\bar E^{t+1}\|_F^2$ as an explicit term in the convergence rate, Dysco's server-fixed merge step collapses its bilinear $\propto\Delta_A\Delta_B$ form to a linear $\propto\rho_A\Delta_B$ form by eliminating the $A$-drift contribution.
This yields a $\Theta(\Delta_A/\rho_A)$ tightening of the iterate's optimization error across both regimes---per-round nonconvex stationarity (\cref{cor:tighter_conv}) and $K$-step local SGD (\cref{thm:agg_K}).
At the standard FL operating point $\eta K G_A=\Theta(\sqrt r)$, this tightening reduces to $\Theta(r)$, which is what Dysco's server-fixed subspace allocation is engineered to exploit and what FedAvg-LoRA's factor averaging cannot.

\section{Additional Synthetic Experiments}
\label{app:synthetic-additional}

This appendix collects additional details and ablations for the synthetic federation introduced in~\cref{subsec:synthetic settings}.

\subsection{Synthetic-Experiment Setup Details}
\label{app:synthetic-setup}

This appendix completes the description of the controlled federation introduced in~\cref{subsec:synthetic settings}.

\paragraph{Data generation.}
The hidden dimension is $d_{\mathrm{model}}=128$, the per-client target is $d_{\mathrm{out}}=10$-dimensional, and the sequence length is $S=8$.
For each client $i$ we draw a column-orthonormal subspace $U_i \in \mathbb{R}^{d_{\mathrm{model}} \times r}$ and sequence latents $Z_i \in \mathbb{R}^{S \times r}$, so that the input tokens $x = Z_i U_i^\top \in \mathbb{R}^{S \times d_{\mathrm{model}}}$ lie entirely in $U_i$.
Labels are produced as $y = \overline{x}\,W_i^{\star\top} \in \mathbb{R}^{d_{\mathrm{out}}}$, where $\overline{x}$ is the token-mean of the sequence $x$ and $W_i^\star = (1-\alpha)\,W_{\mathrm{shared}} + \alpha\,W_{\mathrm{perturb},i}$ with $\alpha=1$ to induce the maximum cross-client heterogeneity.
When $N\!\cdot\!r \le d_{\mathrm{model}}$ the subspaces $\{U_i\}$ form a mutually orthogonal partition of $\mathbb{R}^{d_{\mathrm{model}}}$; when $N\!\cdot\!r > d_{\mathrm{model}}$ each client receives an independent random orthonormal $r$-dim subspace, so inter-client orthogonality is no longer strict.
Each client has $100$ training and $100$ test sequences.

\paragraph{Model architecture.}
Local training uses a depth-$L$ Transformer encoder with $d_{\mathrm{model}}=128$, $h=4$ attention heads, and a feed-forward block of width $256$.
LoRA adapters of rank $r$ are attached to the selected attention projection(s) (any combination of $Q$, $K$, $V$) at every encoder layer.
The trainable factors are the LoRA $A_i,B_i$ matrices; all non-LoRA parameters of the Transformer (FFN, LayerNorm, attention output projection, prediction head) are frozen.
For FedAvg, $A_i$ is resampled i.i.d.\ from $\mathcal{N}(0, 1/d_{\mathrm{model}})$ each round, matching the standard LoRA initialization used in our MIMIC experiments; for Dysco, $A_i$ is set to $U_i^\top$, so Dysco's $A_i$ is fixed across rounds while $B_i$ is trained locally and merged on the server.

\paragraph{Optimization and evaluation.}
Every configuration is trained for $T=100$ communication rounds with $K=10$ local SGD steps per round, learning rate $0.01$, and is repeated for three random seeds (\(\{42,43,44\}\)).
For each setting we report the merged global model's average MSE loss
\[
F(W) = \tfrac{1}{2N}\sum_{i=1}^N \tfrac{1}{m_i}\sum_{x \in \mathcal{D}_i} \|W(x) - y\|_2^2,
\]
where $W(x) \in \mathbb{R}^{d_{\mathrm{out}}}$ is the Transformer's output for input sequence $x$ and $\mathcal{D}_i$ is client $i$'s training set, and the corresponding top-$1$ test accuracy obtained by comparing $\arg\max$ of $W(x)$ to $\arg\max$ of the target $y$.

\subsection{Block Ablation: Q-only and K-only}
\label{app:blocks-qk}

\cref{subsec:block ablation} reports the two settings ($V$-only and $Q\!+\!K\!+\!V$) where attaching LoRA produces a substantial Dysco lift.
For completeness, this appendix shows the other two settings in the same controlled federation: $Q$-only and $K$-only.

\begin{figure}[tbp]
    \centering
    \includegraphics[width=\linewidth]{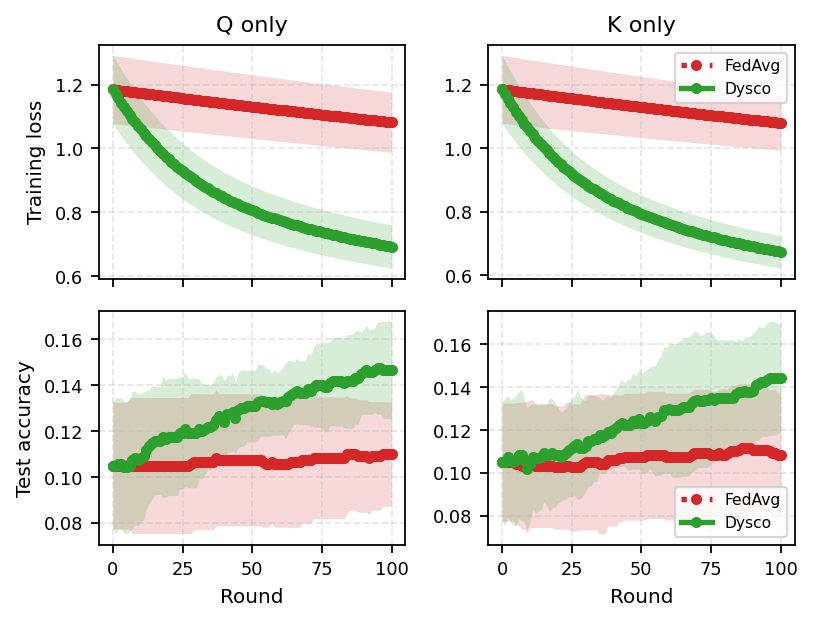}
    \caption{Training loss (top) and test accuracy (bottom) for $Q$-only and $K$-only LoRA configurations. Restricting LoRA to $Q$ or $K$ leaves both methods near random chance, confirming that softmax-mediated attention weights are insensitive to the LoRA basis.}
    \label{fig:synthetic-blocks-qk}
\end{figure}

In both settings, neither method drives the training loss below $0.67$, in sharp contrast to the $V$-only result ($0.08$).
Dysco's final test accuracy is $0.15$ (Q-only) and $0.14$ (K-only) compared to FedAvg's $0.11$ in both, leaving both methods within a small margin of the random-chance floor.
This is consistent with the mechanism described in~\cref{subsec:block ablation}: $Q$ and $K$ only shape attention weights through a softmax that is largely insensitive to the LoRA basis, so a client-aligned subspace allocation has little leverage there.
The data-parameter interference that Dysco neutralizes is dominated by the value projection $V$, which writes into the residual stream.

\subsection{Scaling with Number of Clients: $N=64$}
\label{app:clients-N64}

\cref{subsec:synthetic num clients} reports the main-text settings $N\in\{4,16,100\}$, which span the strict-orthogonal regime and its breakdown.
For completeness, we report here the intermediate setting $N=64$, where $N\!\cdot\!r=256$ is twice the ambient dimension $d=128$ and the random orthonormal fallback is fully engaged.
Dysco's final training loss is $0.25$, compared with $0.55$ for FedAvg, and its test accuracy reaches $0.21$ while FedAvg remains at the random-chance floor near $0.10$.

\begin{figure}[tbp]
    \centering
    \includegraphics[width=0.4\linewidth]{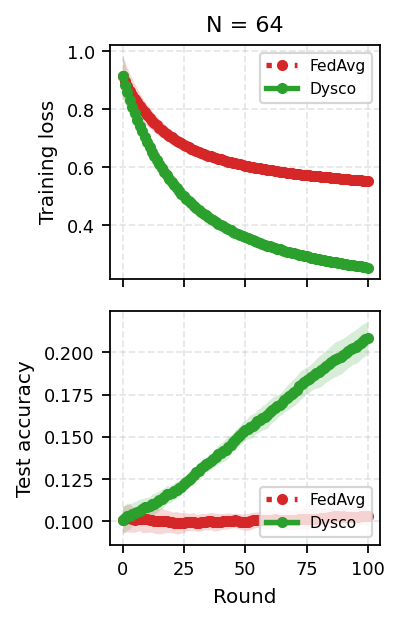}
    \caption{Training loss (top) and test accuracy (bottom) at the boundary case $N=64$ with $r=4$ and $d_{\mathrm{model}}=128$. Dysco retains a clear advantage over FedAvg in both metrics, consistent with the random-orthonormal-fallback analysis in~\cref{subsec:synthetic num clients}.}
    \label{fig:synthetic-clients-N64}
\end{figure}

\subsection{Effect of LoRA Rank: $r=4$}
\label{app:rank-r4}

\cref{subsec:synthetic rank} reports the main-text settings $r\in\{8,16,32\}$, which expose how rank-driven over-parameterization amplifies FedAvg's cross-client interference.
For completeness, we report here the small-capacity setting $r=4$, where the LoRA adapter dimension matches the per-client subspace rank used as the default in the depth, blocks, and clients ablations.
Dysco's final training loss is $0.05$, compared with $0.36$ for FedAvg (a roughly $7\times$ ratio), and its test accuracy reaches $0.55$ while FedAvg sits at $0.20$.

\begin{figure}[tbp]
    \centering
    \includegraphics[width=0.4\linewidth]{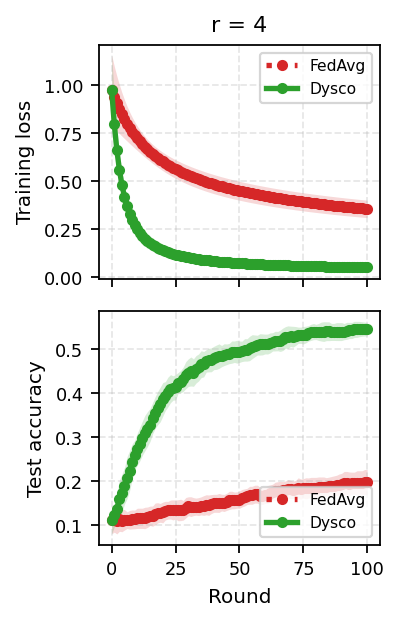}
    \caption{Training loss (top) and test accuracy (bottom) at the small-capacity case $r=4$. Dysco retains its training-loss and test-accuracy advantage over FedAvg, consistent with the rank-scaling analysis in~\cref{subsec:synthetic rank}.}
    \label{fig:synthetic-rank-r4}
\end{figure}

\section{Additional Real-World Experiments}
\label{app:real-world-additional}

This appendix collects additional results on the real-world benchmarks (GLUE and MIMIC-IV-note) introduced in~\cref{sec:experiments}.

\subsection{GLUE Experiments}
\label{sec:appdx-glue}

In this appendix we present the GLUE-benchmark counterpart to the MIMIC-IV results in the main body.
The GLUE setting federates eight semantically diverse NLU tasks (sentiment, paraphrase, entailment, similarity, etc.) over a smaller backbone (RoBERTa-large, $0.4$B), which makes cross-task interference substantially more severe than in the MIMIC setting that shares a common note pool and a larger LLaMA-3.2-1B backbone.
We retain the GLUE experiment here both as a stress-test of Dysco's generality and as a controlled comparison with the same FL algorithms used in~\cref{sec:experiments}.

\paragraph{Dataset and model.}
The GLUE benchmark~\citep{glue} consists of eight datasets, including MRPC~\citep{mrpc}, QQP\footnote{https://data.quora.com/First-Quora-Dataset-Release-Question-Pairs}, QNLI~\citep{qnli}, MNLI~\citep{mnli}, SST-2~\citep{sst2}, CoLA~\citep{cola}, STS-B~\citep{stsb}, and RTE~\citep{rte}.
The tasks involve different types of sentence classification or similarity judgments, so their label spaces and data distributions differ, but they all share some underlying knowledge of language.
Following~\citep{OSRM}, we measure the Matthews correlation coefficient for CoLA and the average of the Pearson and Spearman correlation coefficients for STS-B; accuracy is measured for the remaining tasks.
Each GLUE task is assigned to a different client in a federated setup.
We use RoBERTa-large~\citep{roberta} as the backbone---a high-capacity encoder-only Transformer that builds on BERT~\citep{bert}---to serve as a representative large encoder-only model.

\paragraph{Results.}
\cref{tab:glue results} reports the performance of the global models evaluated on the GLUE benchmark under different federated optimization methods.
Across all five methods, Dysco improves the average GLUE accuracy, with absolute gains of $+2.1$ (FedAvg), $+11.4$ (FedAvgM), $+0.6$ (FedProx), $+1.2$ (Scaffold), and $+8.3$ (FedNova) points.
The largest lifts are concentrated on FedAvgM and FedNova; the smallest is on FedProx, whose average-level gain is partially offset by per-task regressions on STS-B and QNLI.
At the per-task level, the picture is correspondingly varied: Dysco wins on $7/8$ tasks under FedAvgM, $5/8$ under FedAvg, $4/8$ under Scaffold (with two ties), $3/8$ under FedNova (with five ties), and $2/8$ under FedProx, reflecting that an eight-task federation with semantically heterogeneous tasks induces more cross-client interference than the four-task MIMIC federation.
The overall accuracy on GLUE is lower than on MIMIC because GLUE federates eight semantically diverse tasks (entailment, similarity, sentiment, etc.) over a smaller backbone (RoBERTa-large, 0.4B), whereas MIMIC's four related clinical tasks share a common note pool and benefit from a larger model (LLaMA-3.2-1B).
This more severe cross-task interference makes GLUE a substantially harder aggregation problem.
Beyond the final accuracy reported in the table, \cref{fig:acc vs rounds} traces the average GLUE accuracy across communication rounds and shows that Dysco consistently delivers \emph{faster and better} convergence than the baseline.
The Dysco curve rises more steeply in the early rounds, reaches the baseline's final accuracy within roughly half the communication budget, and continues to climb to a noticeably higher plateau by the end of training.
This combination of accelerated convergence and improved final accuracy indicates that suppressing the data-parameter interference both shortens the number of rounds needed to reach a target accuracy and lifts the accuracy ceiling that factor averaging can attain.

\begin{figure}[t]
    \centering
    \includegraphics[width=.65\linewidth]{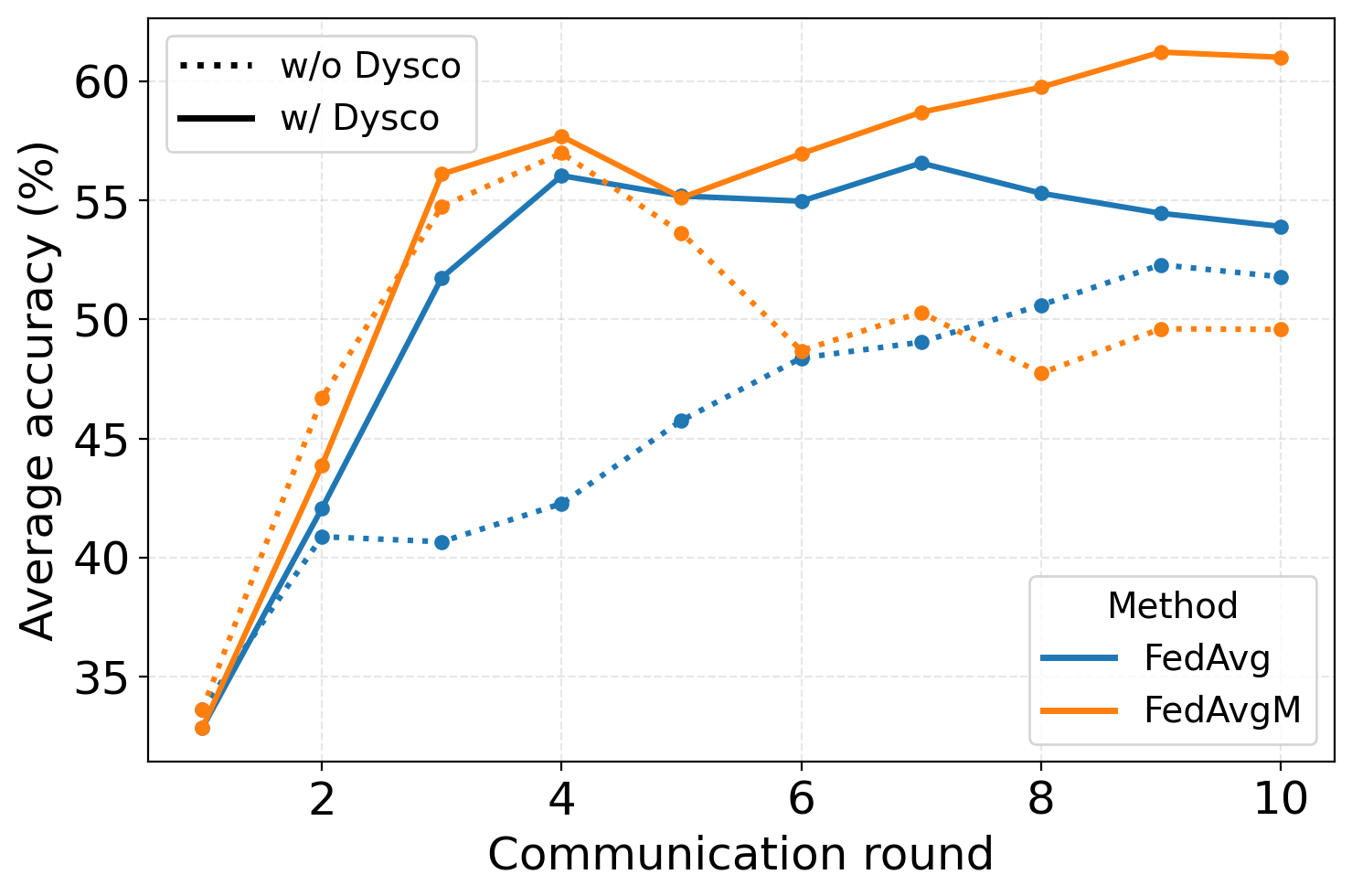}
    \caption{Average accuracy across all GLUE tasks over communication rounds with and without Dysco. Dysco consistently achieves a faster and better convergence compared to the baseline.}
    \label{fig:acc vs rounds}
\end{figure}

\begin{table*}[h]
\small
\centering
\caption{Performance (\%) of the global model on the GLUE benchmark.}
\label{tab:glue results}
\renewcommand{\arraystretch}{1.15}
\resizebox{.88\textwidth}{!}{%
\begin{tabular}{l c cccccccc c}
\toprule
Algorithm & Dysco & MRPC & RTE & STS-B & CoLA & SST-2 & QQP & QNLI & MNLI & Avg. \\
\midrule

\multirow{2}{*}{\centering FedAvg}
 & No  & \textbf{78.9} & 53.8 & 29.8 & 14.4 & 64.8 & \textbf{81.5} & 48.5 & \textbf{42.6} & 51.8 \\
 & Yes & 75.0 & \textbf{61.7} & \textbf{31.7} & \textbf{20.7} & \textbf{79.4} & 79.0 & \textbf{55.9} & 27.8 & \textbf{53.9} \\
\midrule

\multirow{2}{*}{\centering FedAvgM}
 & No  & \textbf{76.0} & 52.0 & 1.7 & 6.9 & 78.9 & 79.5 & 52.7 & 48.9 & 49.6 \\
 & Yes & 70.8 & \textbf{57.4} & \textbf{32.9} & \textbf{36.6} & \textbf{86.6} & \textbf{80.6} & \textbf{62.5} & \textbf{60.5} & \textbf{61.0} \\
\midrule

\multirow{2}{*}{\centering FedProx}
 & No  & \textbf{73.8} & 52.0 & \textbf{$-$33.8} & \textbf{27.8} & 71.3 & \textbf{61.1} & \textbf{44.8} & \textbf{46.2} & 42.9 \\
 & Yes & 73.0 & \textbf{71.8} & $-$49.7 & 25.5 & \textbf{83.8} & 59.6 & 41.0 & 42.7 & \textbf{43.5} \\
\midrule

\multirow{2}{*}{\centering Scaffold}
 & No  & 31.6 & \textbf{52.7} & $-$0.9 & 0.0 & 49.1 & \textbf{63.2} & 49.5 & 32.7 & 34.7 \\
 & Yes & \textbf{68.4} & 47.3 & \textbf{0.7} & 0.0 & \textbf{50.9} & 36.8 & \textbf{50.5} & 32.7 & \textbf{35.9} \\
\midrule

\multirow{2}{*}{\centering FedNova}
 & No  & 31.6 & 52.7 & $-$16.8 & 0.0 & 49.5 & 63.2 & 49.5 & 31.6 & 32.7 \\
 & Yes & 31.6 & 52.7 & \textbf{45.8} & 0.0 & \textbf{49.9} & 63.2 & 49.5 & \textbf{35.2} & \textbf{41.0} \\

\bottomrule
\end{tabular}%
}
\end{table*}

\subsection{Ablation Study}
\label{app:ablation}

We ablate the two core components of Dysco---subspace merging and subspace boosting---by disabling each one individually while keeping the other active.
Specifically, \textbf{w/o Merging} skips cross-client subspace aggregation so that each client initializes LoRA $A$ from its own SVD-computed subspace, and \textbf{w/o Boosting} computes and aggregates subspaces only at round 0 and reuses them in all subsequent rounds without recomputation.
All experiments use FedAvg as the FL algorithm and report the global model performance at round 10.

\begin{table}[h]
\centering
\small
\caption{Ablation results (\%) on the GLUE benchmark (8 tasks, round 10).}
\label{tab:ablation-glue}
\resizebox{.7\linewidth}{!}{
\begin{tabular}{l c c c c}
\toprule
Task & FedAvg & w/o Merging & w/o Boosting & Dysco \\
\midrule
MRPC     & \textbf{78.9} & 77.7 & 77.7 & 75.0 \\
RTE      & 53.8 & 58.5 & 56.7 & \textbf{61.7} \\
STS-B    & 29.8 & $-$60.1 & $-$58.1 & \textbf{31.7} \\
CoLA     & 14.4 & \textbf{36.9} & 35.7 & 20.7 \\
SST-2    & 64.8 & \textbf{90.8} & \textbf{90.8} & 79.4 \\
QQP      & \textbf{81.5} & 72.7 & 65.1 & 79.0 \\
QNLI     & 48.5 & 69.5 & \textbf{70.4} & 55.9 \\
MNLI     & 42.6 & 38.7 & \textbf{44.5} & 27.8 \\
\midrule
Avg.     & 51.8 & 48.1 & 47.8 & \textbf{53.9} \\
\bottomrule
\end{tabular}
}
\end{table}

\begin{table}[h]
\centering
\small
\caption{Ablation results (\%) on the MIMIC-IV-note dataset (4 tasks, round 10).}
\label{tab:ablation-mimic}
\begin{tabular}{l c c c c}
\toprule
Task & FedAvg & w/o Merging & w/o Boosting & Dysco \\
\midrule
LIVER     & \textbf{86.8} & 85.4 & 86.5 & 86.6 \\
COAG      & 79.0 & 79.6 & 80.3 & \textbf{80.7} \\
NEURO     & 83.2 & 85.7 & \textbf{87.0} & 85.2 \\
WGHTLOSS  & 74.3 & \textbf{78.5} & 76.7 & 77.5 \\
\midrule
Avg.      & 80.8 & 82.3 & \textbf{82.6} & 82.5 \\
\bottomrule
\end{tabular}
\end{table}

\paragraph{Discussion.}
On both benchmarks, the full Dysco achieves the best average performance on GLUE (53.9\%) and is competitive on MIMIC (82.5\% vs.\ 82.6\% for w/o Boosting), while the two ablation variants and FedAvg perform worse on average.
Removing merging or boosting individually still improves over vanilla FedAvg on MIMIC, indicating that each component provides independent benefit.
However, on GLUE the ablation variants underperform FedAvg in average accuracy despite excelling on individual tasks (e.g., SST-2, QNLI), because the missing component causes severe degradation on STS-B.
This confirms that subspace merging and boosting are complementary: merging aligns cross-client subspaces to reduce interference, while boosting adaptively refines them across rounds to maintain representational quality.

\section{Additional LoRA-Rank Experiments on Real Data}

We further investigate the effect of LoRA ranks in the two real-world benchmarks.

\begin{table*}[t]
\small
\centering
\caption{Performance (\%) of the global model on the GLUE benchmark with $r=16$.}
\label{tab:appdx-lora-rank-16-glue}
\resizebox{.7\textwidth}{!}{%
\begin{tabular}{ccccccccccc}
\toprule
Algorithm & Method & MRPC & RTE & STS-B & CoLA & SST-2 & QQP & QNLI & MNLI & Avg. \\
\midrule
\multirow{2}{*}{\centering FedAvg}
& Baseline & 31.6 & 52.7 & -74.0 & 0.0 & 49.1 & \textbf{63.2} & 49.5 & 35.4 & 25.9 \\
& Ours & \textbf{73.5} & \textbf{63.2} & \textbf{53.8} & \textbf{23.5} & \textbf{90.1} & 59.9 & \textbf{76.7} & \textbf{38.7} & \textbf{59.9} \\
\midrule
\multirow{2}{*}{\centering FedAvgM}
 & Baseline & \textbf{72.8} & \textbf{66.1} & -12.0 & \textbf{22.7} & \textbf{88.8} & 60.2 & \textbf{77.6} & 19.0 & \textbf{49.4} \\
 & Ours & 31.6 & 52.7 & \textbf{-12.0} & 0.0 & 49.1 & \textbf{63.2} & 49.5 & \textbf{35.4} & 33.7 \\
\midrule
\multirow{2}{*}{\centering FedProx}
 & Baseline & 70.8 & \textbf{66.8} & 7.6 & 21.9 & 88.3 & \textbf{59.7} & 75.3 & 24.1 & 51.8 \\
 & Ours & \textbf{73.3} & 66.1 & \textbf{29.6} & \textbf{39.7} & \textbf{90.7} & 58.8 & \textbf{76.6} & \textbf{24.6} & \textbf{57.4} \\
\midrule
\multirow{2}{*}{\centering Scaffold}
 & Baseline & 31.6 & \textbf{52.7} & \textbf{2.6} & 0.0 & 49.1 & \textbf{63.2} & 49.5 & 31.8 & 35.1 \\
 & Ours & \textbf{68.4} & 47.3 & 1.4 & 0.0 & \textbf{50.9} & 36.8 & \textbf{50.5} & 31.8 & \textbf{35.9} \\
\midrule
\multirow{2}{*}{\centering FedNova}
 & Baseline & \textbf{77.0} & 60.3 & -49.5 & 23.2 & \textbf{91.5} & \textbf{73.9} & 65.5 & \textbf{20.7} & 45.3 \\
 & Ours & 75.5 & \textbf{70.0} & \textbf{-36.8} & \textbf{39.7} & 91.2 & 71.7 & \textbf{82.3} & 15.6 & \textbf{51.2} \\
\bottomrule
\end{tabular}}
\end{table*}

\begin{table*}[t]
\small
\centering
\caption{Performance (\%) of the global model on the GLUE benchmark with $r=32$.}
\label{tab:appdx-lora-rank-32-glue}
\resizebox{.7\textwidth}{!}{%
\begin{tabular}{ccccccccccc}
\toprule
Algorithm & Method & MRPC & RTE & STS-B & CoLA & SST-2 & QQP & QNLI & MNLI & Avg. \\
\midrule
\multirow{2}{*}{\centering FedAvg}
 & Baseline & 31.6 & 52.7 & -17.0 & 0.0 & 49.1 & \textbf{63.2} & 49.5 & \textbf{35.4} & 33.1 \\
 & Ours & \textbf{67.6} & \textbf{69.0} & \textbf{20.4} & \textbf{32.9} & \textbf{91.1} & 61.0 & \textbf{78.1} & 25.9 & \textbf{55.7} \\
\midrule
\multirow{2}{*}{\centering FedAvgM}
 & Baseline & \textbf{69.9} & \textbf{69.7} & \textbf{55.8} & \textbf{34.6} & \textbf{87.5} & 55.6 & \textbf{78.1} & 21.0 & \textbf{59.0} \\
 & Ours & 31.6 & 52.7 & -8.0 & 0.0 & 49.1 & \textbf{63.2} & 49.5 & \textbf{35.4} & 34.2 \\
\midrule
\multirow{2}{*}{\centering FedProx}
 & Baseline & 67.9 & \textbf{67.5} & \textbf{2.6} & 18.3 & \textbf{90.8} & 47.5 & 73.8 & 19.1 & 48.5 \\
 & Ours & \textbf{73.0} & 65.3 & -4.8 & \textbf{38.7} & 88.6 & \textbf{67.9} & \textbf{75.2} & \textbf{19.6} & \textbf{53.0} \\
\midrule
\multirow{2}{*}{\centering Scaffold}
 & Baseline & 31.6 & 52.7 & 3.0 & 0.0 & 49.1 & 63.2 & 49.5 & 35.4 & 35.6 \\
 & Ours & 31.6 & 52.7 & \textbf{3.0} & 0.0 & 49.1 & 63.2 & 49.5 & 35.4 & \textbf{35.6} \\
\midrule
\multirow{2}{*}{\centering FedNova}
 & Baseline & 73.5 & \textbf{66.4} & \textbf{58.6} & 31.7 & 89.7 & \textbf{70.3} & \textbf{77.1} & 20.1 & \textbf{60.9} \\
 & Ours & \textbf{77.9} & 63.9 & 33.4 & \textbf{39.8} & \textbf{90.9} & 69.0 & 76.8 & \textbf{30.6} & 60.3 \\
\bottomrule
\end{tabular}}
\end{table*}

\begin{table*}[t]
\centering
\caption{Performance (\%) of the global model on the MIMIC-IV-note dataset with $r=16$.}
\label{tab:appdx-lora-rank-16-mimic}
\begin{tabular}{ccccccc}
\toprule
Algorithm & Method & LIVER & COAG & NEURO & WGHTLOSS & Avg. \\
\midrule
\multirow{2}{*}{\centering FedAvg}
 & Baseline & \textbf{86.7} & \textbf{81.7} & 86.0 & \textbf{80.6} & \textbf{83.8} \\
 & Ours & 86.3 & 81.0 & \textbf{86.6} & 79.2 & 83.3 \\
\midrule
\multirow{2}{*}{\centering FedAvgM}
 & Baseline & 87.4 & 81.4 & 86.2 & \textbf{81.9} & 84.2 \\
 & Ours & \textbf{88.1} & \textbf{83.8} & \textbf{88.2} & 81.6 & \textbf{85.4} \\
\midrule
\multirow{2}{*}{\centering FedProx}
 & Baseline & 86.1 & 80.5 & 84.0 & 78.8 & 82.3 \\
 & Ours & \textbf{86.2} & \textbf{82.1} & \textbf{84.9} & \textbf{81.4} & \textbf{83.6} \\
\midrule
\multirow{2}{*}{\centering Scaffold}
 & Baseline & 49.5 & 47.7 & 49.1 & 51.0 & 49.3 \\
 & Ours & \textbf{55.2} & \textbf{52.0} & \textbf{50.5} & \textbf{51.3} & \textbf{52.3} \\
\midrule
\multirow{2}{*}{\centering FedNova}
 & Baseline & \textbf{87.8} & \textbf{82.4} & \textbf{86.7} & 80.0 & \textbf{84.2} \\
 & Ours & 86.6 & 81.2 & 85.5 & \textbf{82.1} & 83.8 \\
\bottomrule
\end{tabular}
\end{table*}

\begin{table*}[t]
\centering
\caption{Performance (\%) of the global model in the MIMIC-IV-note dataset with $r=32$.}
\label{tab:appdx-lora-rank-32-mimic}
\begin{tabular}{ccccccc}
\toprule
Algorithm & Method & LIVER & COAG & NEURO & WGHTLOSS & Avg. \\
\midrule
\multirow{2}{*}{\centering FedAvg}
 & Baseline & \textbf{86.6} & \textbf{80.7} & 85.8 & 78.7 & \textbf{83.0} \\
 & Ours & 86.3 & 78.9 & \textbf{85.9} & \textbf{79.1} & 82.6 \\
\midrule
\multirow{2}{*}{\centering FedAvgM}
 & Baseline & \textbf{88.1} & 81.6 & \textbf{87.7} & \textbf{80.5} & \textbf{84.5} \\
 & Ours & 87.8 & \textbf{81.7} & 87.3 & 80.2 & 84.2 \\
\midrule
\multirow{2}{*}{\centering FedProx}
 & Baseline & \textbf{86.2} & 78.4 & 84.0 & 75.7 & 81.1 \\
 & Ours & 85.7 & \textbf{80.3} & 84.0 & \textbf{76.4} & \textbf{81.6} \\
\midrule
\multirow{2}{*}{\centering Scaffold}
 & Baseline & \textbf{50.0} & \textbf{50.3} & \textbf{51.9} & \textbf{51.7} & \textbf{51.0} \\
 & Ours & 48.2 & 48.7 & 49.8 & 48.0 & 48.7 \\
\midrule
\multirow{2}{*}{\centering FedNova}
 & Baseline & 86.3 & 79.6 & 85.4 & \textbf{79.7} & 82.7 \\
 & Ours & \textbf{88.1} & \textbf{81.7} & \textbf{85.6} & 79.2 & \textbf{83.7} \\
\bottomrule
\end{tabular}
\end{table*}

\subsection{Detailed Computational Overhead}
\label{subsec:appdx-overhead-detail}

\cref{tab:overhead-detail} provides the per-setting wall-clock time breakdown for the overhead experiment summarized in~\cref{tab:overhead-avg}.
All experiments use a single outer round to isolate per-round cost.
Settings span the MIMIC-IV-note benchmark (4-task, Llama-3.2-1B) and the GLUE benchmark (2/4/6/8-task, RoBERTa-large).

\begin{table*}[t]
\centering
\caption{Detailed wall-clock time per round (seconds) for each FL method and setting, comparing Baseline and Dysco.  All experiments use 1 outer round.}
\label{tab:overhead-detail}
\begin{tabular}{llccc}
\toprule
Method   & Setting      & Baseline (s) & Dysco (s) & Overhead (\%) \\
\midrule
\multirow{5}{*}{FedAvg}
  & MIMIC 4-task & 2890.9 & 2902.4 & +0.4 \\
  & GLUE 2-task  &  240.6 &  246.7 & +2.5 \\
  & GLUE 4-task  &  591.7 &  615.8 & +4.1 \\
  & GLUE 6-task  & 1675.5 & 1679.4 & +0.2 \\
  & GLUE 8-task  & 2220.2 & 2224.8 & +0.2 \\
\midrule
\multirow{5}{*}{FedProx}
  & MIMIC 4-task & 2884.2 & 2910.2 & +0.9 \\
  & GLUE 2-task  &  247.2 &  253.9 & +2.7 \\
  & GLUE 4-task  &  655.9 &  641.9 & $-$2.1 \\
  & GLUE 6-task  & 1796.9 & 1814.8 & +1.0 \\
  & GLUE 8-task  & 2354.8 & 2403.7 & +2.1 \\
\midrule
\multirow{5}{*}{Scaffold}
  & MIMIC 4-task & 2869.7 & 2905.0 & +1.2 \\
  & GLUE 2-task  &  287.2 &  275.6 & $-$4.0 \\
  & GLUE 4-task  &  600.7 &  620.2 & +3.2 \\
  & GLUE 6-task  & 1704.2 & 1724.0 & +1.2 \\
  & GLUE 8-task  & 2224.2 & 2255.8 & +1.4 \\
\bottomrule
\end{tabular}
\end{table*}

\end{document}